\documentclass{article}

\usepackage[square,numbers,sort&compress]{natbib}
\usepackage[preprint]{neurips_2024}

\usepackage[utf8]{inputenc} 
\usepackage[T1]{fontenc}    
\usepackage[hidelinks]{hyperref}       
\usepackage{url}            
\usepackage{booktabs}       
\usepackage{amsfonts}       
\usepackage{nicefrac}       
\usepackage{microtype}      
\usepackage{xcolor}         
\usepackage{amsthm}
\usepackage{amsmath}
\usepackage{float}
\newtheorem{definition}{Definition}
\newtheorem{proposition}{Proposition}
\newtheorem{theorem}{Theorem}
\newtheorem{lemma}{Lemma}
\newtheorem{remark}{Remark}
\newtheorem{assumption}{Assumption}
\newtheorem{corollary}{Corollary}

\usepackage{graphicx} 
\usepackage{subcaption} 

\newcommand{\PLH}{{\mkern-2mu\times\mkern-2mu}}

\title{Taming Score-Based Diffusion Priors for Infinite-Dimensional Nonlinear Inverse Problems}

\author{%
  Lorenzo Baldassari \\
  University of Basel \\
  \texttt{\texttt{ lorenzo.baldassari@unibas.ch}} \\
  \And
  Ali Siahkoohi \\
  Rice University \\
  \texttt{alisk@rice.edu} \\
  \And
  Josselin Garnier\\
  Ecole Polytechnique, IP Paris\\
  \texttt{ josselin.garnier@polytechnique.edu}\\
  \And
  Knut S\o{}lna\\
  University of California Irvine\\
  \texttt{ksolna@math.uci.edu}\\
  \And
  Maarten V. de Hoop\\
  Rice University\\
  \texttt{mvd2@rice.edu}\\
}

\begin{document}

\maketitle

\begin{abstract}
  This work introduces a sampling method capable of solving Bayesian inverse problems in function space. It does not assume the log-concavity of the likelihood, meaning that it is compatible with nonlinear inverse problems. The method leverages the recently defined infinite-dimensional score-based diffusion models as a learning-based prior, while enabling provable posterior sampling through a Langevin-type MCMC algorithm defined on function spaces. A novel convergence analysis is conducted, inspired by the fixed-point methods established for traditional regularization-by-denoising algorithms and compatible with weighted annealing. The obtained convergence bound explicitly depends on the approximation error of the score; a well-approximated score is essential to obtain a well-approximated posterior. 
  Stylized and PDE-based examples are provided, demonstrating the validity of our convergence analysis. We conclude by presenting a discussion of the method's  challenges related to learning the score and computational complexity.





\end{abstract}

\section{Introduction}

Solving inverse problems is a central challenge in many applications. The objective is to estimate unknown parameters using noisy observations or measurements \cite{tarantola2005inverse}. One of the main challenges of inverse problems is that they are often ill-posed: they may have no solutions, two or more solutions, or the solution may be unstable with respect to small errors in the observations \cite{hadamard2003lectures}. However, by framing an inverse problem within a probabilistic framework known as Bayesian inference, one can characterize all possible solutions \cite{lehtinen1989linear, tarantola2005inverse, stuart2014uncertainty}. In the Bayesian approach, we first define a prior probability distribution that describes our knowledge on the unknown before any measurements are taken, along with a model for the observational noise. The objective is to estimate the \emph{posterior distribution}, which characterizes the distribution of the unknown given noisy measurements. One can then sample from the posterior to extract statistical information for uncertainty quantification \cite{stuart2010inverse, knapik2011bayesian, dashti2011uncertainty, stuart2014uncertainty}. 

Generative machine learning models offer a powerful way to compute the posterior. In this context, diffusion-based methods \cite{sohl2015deep, ho2020denoising} have become the preferred choice due to their ability to generate high-quality data in a variety of tasks \cite{dhariwal2021diffusion, saharia2022photorealistic, ho2022video, yang2023diffusion, xu2022geodiff, yang2023diffusionsurvey}. These methods typically perturb finite-dimensional data with a sequence of noise and then learn to reverse this corruption \cite{ho2020denoising}. A continuous-time generalization of these methods, enabling new sampling procedures and further extending their capabilities, has been subsequently proposed by \citet{song2020score}, who introduced a unified framework based on stochastic differential equations (SDEs), often referred to as the \emph{score-based diffusion model} (SDM). Instead of perturbing data with a finite number of noise distributions at discrete times, \citet{song2020score} have considered a continuum of distributions that evolve in time according to a diffusion process whose dynamics is described by an SDE. Crucially, the reverse process is also a diffusion  \cite{anderson1982reverse} satisfying a reverse-time SDE whose drift depends on the logarithmic gradient of the noised data density---the \emph{score function}---which can be estimated through a neural network by leveraging the connection
between diffusion models and score matching \cite{vincent2011connection, song2019generative, song2020improved, song2020score}. 

After their introduction, SDMs have been widely utilized for solving inverse problems in a Bayesian fashion \cite{song2021solving, feng2023score, song2020score, chung2022diffusion, song2022pseudoinverse, feng2023efficient, batzolis2021conditional, jalal2021robust, xu2024provably, dhariwal2021diffusion, kawar2021snips}. 
Some have proposed to sample from the posterior using SDMs conditioned on observations \cite{batzolis2021conditional, song2020score, dhariwal2021diffusion, kawar2021snips, jalal2021robust}. Others, however, have suggested utilizing the learned score of the prior distribution, that is, the so-called unconditional score model \cite{song2021solving, feng2023efficient, feng2023score, sun2023provable, xu2024provably}. Currently, there are two ways the learned score of the prior is employed to sample the posterior distribution of the inverse problem: (i) modifying the unconditional reverse diffusion process of the pretrained score-based model, which initially produces samples from the prior distribution, by conditioning on the observed data so that the modified reverse diffusion process yields samples from the posterior distribution \cite{song2021solving}; and (ii) using existing Markov Chain Monte Carlo (MCMC) samplers where the learned score replaces the score of the prior distribution \cite{feng2023score, sun2023provable, xu2024provably}. 

Regardless of whether they utilize the conditional or unconditional score, all the works cited above have something in common: they assume that the posterior is \emph{supported on a finite-dimensional space}. However, in many inverse problems, especially those governed by partial differential equations (PDEs), the unknown parameters to be estimated are \emph{functions} that exist in a suitable function space, typically an infinite-dimensional Hilbert space \cite{uhlmann2016inverse, dynin1978inversion, stuart2010inverse}. Unfortunately, discretizing the input and output functions into finite-dimensional vectors and utilizing SDMs to sample from the posterior is not always desirable, as theoretical results of current diffusion models 
do not generalize well to large-scale inverse problems \cite{block2020generative, chen2022sampling, sun2023provable, xu2024provably}.
In the last year, however, some progress has been made to address these concerns. Building upon the theory of infinite-dimensional stochastic analysis \cite{follmer1986time, millet1989time, da2006introduction, da2014stochastic}, SDMs have been extended to operate directly in function spaces \cite{kerrigan2022diffusion, lim2023score, franzese2024continuous, pidstrigach2023infinite, hagemann2023multilevel, bond2023infty, lim2024score}. Some works have started employing infinite-dimensional SDMs to solve inverse problems, providing a discretization-invariant numerical platform for exploring the posterior \cite{baldassari2024conditional, pidstrigach2023infinite, hosseini2023conditional}. However, the theoretical guarantees provided by these works require the inverse problem to be linear, whereas many interesting inverse problems are \emph{nonlinear}, like those arising in electrical impedance tomography \cite{calderon2006inverse, borcea2002electrical, uhlmann2009electrical}, data assimilation \cite{law2015data}, photo-acoustic tomography \cite{bal2011multi, bal2010inverse}, and boundary rigidity \cite{kachalov2001inverse}. 

In this work, we take a first step towards bridging this gap and provide a method incorporating SDMs that provides \emph{theoretical guarantees for sampling the posterior of nonlinear inverse problems in function spaces}. In doing so, we follow the approach of \cite{feng2023score, sun2023provable, xu2024provably}, who employ the learned score as a prior within an MCMC scheme. Specifically, we extend the theoretical setup of \cite{sun2023provable} to function spaces. Our method utilizes the infinite-dimensional unconditional SDM as a powerful learning-based prior, 
while enabling provable posterior sampling through a Langevin-type MCMC algorithm \cite{roberts1996exponential} for functions. Most importantly, we present a convergence analysis with accompanying error bounds that are \emph{dimension-free} and compatible with weighted annealing, a heuristic often used to mitigate issues like slow convergence and mode collapse in Langevin algorithms \cite{kirkpatrick1983optimization, neal2001annealed}. The main feature of our convergence analysis is that it \emph{does not require the log-concavity of the likelihood}, meaning that it is compatible with nonlinear inverse problems. We conclude this work by presenting stylized and PDE-based examples demonstrating the validity of our convergence analysis, while also providing additional insights into the method's challenges, which are common within the SDM literature \cite{chen2022sampling}.

\paragraph{Related Work} Since the method we propose employs the learned score as a prior in a Langevin-type MCMC posterior sampler, our work combines elements from three contemporary research areas: MCMC methods for functions, Bayesian nonlinear inverse problems, and score-based diffusion.

There exists a large body of literature on infinite-dimensional MCMC algorithms \cite{beskos2017geometric, wallin2018infinite, durmus2019high, durmus2017nonasymptotic, dalalyan2017theoretical, hairer2014spectral, cotter2013mcmc, cui2016dimension, cui2024multilevel, beskos2018multilevel, morzfeld2019localization, muzellec2022dimension, beskos2008mcmc}. However, the main inspiration behind our work is the non-asymptotic stationary convergence analysis recently developed in the finite-dimensional setting by \citet{sun2023provable} for a method known by various names, such as the plug-and-play unadjusted Langevin algorithm (PnP-ULA) or plug-and-play Monte Carlo (PMC-RED), with the latter making the connection to regularization-by-denoising algorithms \cite{reehorst2018regularization,romano2017little} explicit. This method employs a plug-and-play approach within an iterative Monte Carlo scheme; specifically, it aims to learn an approximation of the prior density through a denoising algorithm while keeping an explicit likelihood, in the same spirit as fixed-point algorithms \cite{buzzard2018plug}. While similar methods have been used frequently in the past \cite{venkatakrishnan2013plug, alain2014regularized, guo2019agem, kadkhodaie2021stochastic}, a rigorous proof of convergence in the context of stochastic Bayesian algorithms was only recently proposed by \citet{laumont2022bayesian}. However, their proof assume strong conditions on the forward model of the inverse problem, whereas \citet{sun2023provable} rely on weaker conditions, meaning that their analysis is compatible with nonlinear inverse problems and weighted annealing. Unfortunately, the convergence bound in \citet{sun2023provable} is not dimension-free and becomes uninformative in the limit of infinite dimension. In our work, we fill this gap by \emph{carrying out the whole convergence analysis directly in function spaces}.

Defining a method that provably samples from a non-log-concave posterior distribution, especially in function spaces, is a well-known challenge since it results in a high-dimensional, non-convex problem. Recently, a series of rigorous mathematical papers, mostly by Richard Nickl and his collaborators, have approached nonlinear inverse problems within a probabilistic framework \cite{nickl2022polynomial, nickl2019bernstein, nickl2020bernstein, abraham2019nonparametric, furuya2024consistency, giordano2020consistency, bohr2021log, paternain2012attenuated, monard2021consistent, bonito2017diffusion, nickl2021some, vershynin2018high, nickl2020convergence, nickl2017nonparametric, monard2021statistical, spokoiny2019bayesian}; see \cite{nickl2023bayesian} for an overview. The general idea is to provide a set of assumptions for the forward model to mitigate the non-log-concavity of the posterior.
The main concerns of these works are ensuring \emph{statistical consistency}, i.e., that the posterior concentrates most of its mass around the actual parameter that generated the data, and \emph{computability}. For the former, the global stability of the inverse problem appears to be a sufficient condition. While we have not addressed this in our work, it can be imposed by restricting the family of nonlinear inverse problems under consideration, thus without changing the essence of our convergence analysis. A stronger assumption---local gradient stability of the forward map---is crucial for computability, as it ensures local log-concavity of the posterior. This implies that if the Langevin MCMC algorithm is initialized in such a local region, proving convergence and fast mixing time of the sampling procedure becomes easier.
We discuss the challenges related to the computational complexity of our method in the Discussion and Conclusion section; it's worth mentioning, however, that in our work we focus only on convergence,
even though our theoretical setup, being compatible with weighted annealing, provides a heuristic to speed up the mixing of the Markov chain. Our analysis then can be seen as a first step toward relaxing the strong assumptions on the forward model of the aforementioned works.
 
In our convergence analysis, the learned score plays a key role. Among the theoretical frameworks proposed to define SDMs in infinite dimensions, we consider the one by \citet{pidstrigach2023infinite} and \citet{baldassari2024conditional}. Finally in our work we show that the obtained convergence bound explicitly depends on the $H$-accuracy of the approximated score, where $H$ is the infinite-dimensional separable Hilbert space to which the unknown parameter belongs. We prove convergence even for an \emph{imperfect score}, though a well-approximated score is necessary to obtain a well-approximated posterior. This assumption improves upon those commonly used for sampling non-log-concave distributions, which often rely on stronger $L^\infty$-accuracy \cite{de2021diffusion}, aligning our analysis with that of \citet{chen2022sampling}.

\paragraph{Our Contribution} In this work, we provide \emph{theoretical guarantees for a method incorporating score-based diffusion models that samples from the posterior of nonlinear inverse problems in function spaces}. Specifically, the main contributions of this work are as follows:
\begin{itemize}
\item We define a method that utilizes infinite-dimensional SDMs as a learning-based prior within a Langevin-type MCMC scheme for functions (Section \ref{sec:method}).
\item We extend the non-asymptotic stationary convergence analysis of \citet{sun2023provable} to infinite dimensions. We prove that our algorithm  converges to the posterior under possibly non-log-concave likelihoods and imperfect prior scores. The obtained convergence bound is dimension-free and depends explicitly on the score mismatch and the network approximation errors  (subsection \ref{sec:main}).
\item We apply our method to two examples. In the first, we illustrate the accuracy of our method by applying it to a simple benchmark distribution. In the second example, we sample from the posterior of a complex nonlinear PDE-based inverse problem, where the forward operator is the acoustic wave equation (Section \ref{sec:num}).
\item We discuss the method's challenges related to learning the score and computational complexity in the  Discussion and Conclusion section. 
\end{itemize}

\section{Background}
\subsection{The Bayesian Approach to Inverse Problems}
We consider the possibly nonlinear inverse problem
\begin{equation}
\textbf{y} = \mathcal{A}(X_0) + \textbf{b},
\label{def:IP}
\end{equation}
where the unknown parameter $X_0$ is modeled as an $H$-valued random variable and $H$ is an infinite-dimensional separable Hilbert
space, $\mathcal{A} : H \to \mathbb{R}^N$ is the measurement operator, and $\textbf{b}$ is the noise term with a given density $\rho$ with respect to the Lebesgue measure over $\mathbb{R}^N$. We assume to have some prior knowledge about the distribution of $X_0$ before any measurements are taken. This knowledge is encoded in a given prior measure $\mu_0$. The solution to \eqref{def:IP} is then represented by the conditional probability measure of $X_0|\textbf{y} \sim \mu^\textbf{y}$, which is typically referred to as the posterior \cite{stuart2010inverse}. 
If $\mathbb{E}_{\mu_0}[ \rho(\textbf{y}-\mathcal{A}(X_0))] <+\infty$, which is the case for instance when the density $\rho$ is bounded (such as a multivariate Gaussian $\mathcal{N}(0,\mathbf{\Gamma}$)), then $\mu^\textbf{y}$ is absolutely continuous with respect to $\mu_0$ (we write $\mu^\textbf{y}\ll \mu_0$) and its Radon-Nikodym derivative is given by \begin{equation}
\frac{d \mu^\textbf{y}}{d\mu_0}(X) = \frac{1}{Z(\textbf{y})} \text{exp}({-\Phi_0(X;\textbf{y})}),
\label{def:posterior}
\end{equation}
where $\Phi_0(X; \textbf{y}) := - \log\big(
\rho(\textbf{y}-\mathcal{A}(X))\big)$ is the negative log-likelihood.
Explicitly characterizing $\mu^\textbf{y}$ in \eqref{def:posterior} is challenging, particularly in high dimensions, due to the intractable normalizing constant $Z(\textbf{y}):= \int_H \text{exp}({-\Phi_0(X;\textbf{y})})  d\mu_0(X)$. Popular methods for exploring the posterior, such as Langevin-type MCMC algorithms, aim to generate samples distributed approximately according to $\mu^\textbf{y}$. As anticipated in the Introduction section, we propose to extend one such algorithms to infinite dimensions: the plug-and-play Monte Carlo method (PMC-RED) proposed by \citet{sun2023provable}. 

\subsection{Formulation of PMC-RED}
For the reader's convenience, we will now review the formulation of PMC-RED proposed by \citet{sun2023provable} for sampling the posterior of a possibly nonlinear imaging inverse problem in finite dimensions, $\textbf{y} = \textbf{A}(\textbf{x})+\textbf{e}$, with $\textbf{A}:\mathbb{R}^n \to \mathbb{R}^m$ and $\textbf{e} \sim \mathcal{N}(0,\sigma^2\textbf{I})$. 

PMC-RED is built on the fusion of traditional regularization-by-denoising (RED) algorithms \cite{reehorst2018regularization,romano2017little} and score-based generative modelling \cite{ho2020denoising, song2019generative, song2020score}. It incorporates expressive score-based generative priors in a plug-and-play fashion \cite{venkatakrishnan2013plug, alain2014regularized, guo2019agem, kadkhodaie2021stochastic} for conducting provable posterior sampling. Given an initial state $\textbf{x}_0 \in \mathbb{R}^n$, PCM-RED is defined as the following recursion
\begin{equation}
\textbf{x}_{k+1} = \textbf{x}_{k} - \gamma \big( \nabla g(\textbf{x}_k) - \textbf{S}_\theta(\textbf{x}_k,\sigma)\big) + \sqrt{2\gamma} \textbf{Z}_k, 
\label{eq:PMC-RED-sun}
\end{equation}
where $\textbf{Z}_k = \int_k^{k+1} d \textbf{W}_t$ follows the $m$-dimensional i.i.d normal distribution, $\{\textbf{W}_t\}_{t\geq0}$ represents the $m$-dimensional Brownian motion, $\gamma>0$ denotes the step-size, $g$ is the negative log-likelihood, and $\textbf{S}_\theta(\textbf{x}_k,\sigma) \approx \nabla \log p_{\sigma} (\textbf{x}_k)$ is the score network for $p_\sigma$, a smoothed prior with $\nabla \log p_\sigma \to \nabla \log p $ as $\sigma \to 0$.  
A motivation for using $p_\sigma$ is that $p$ may be non-differentiable, precluding the use of algorithms such as gradient descent for maximum a posteriori (MAP) estimation. This motivates the application of proximal methods \cite{beck2009fast, boyd2011distributed} like RED \cite{beck2009fastIEE}. 
Interestingly, \citet{sun2023provable} notice that since the gradient-flow ODE links RED to the Langevin diffusion described by the SDE
\[
d\textbf{x}_t = \big (\nabla \log p(\textbf{x}_t)- \nabla g(\textbf{x}_t) \big ) dt + \sqrt{2} d \textbf{W}_t,
\]one can interpret \eqref{eq:PMC-RED-sun} as a parallel MCMC algorithm of RED for posterior sampling.
Indeed, PMC-RED aligns with the common Langevin MCMC algorithm \cite{laumont2022bayesian}. However, \citet{sun2023provable} establish a novel convergence analysis that is compatible with non-log-concave likelihoods and weighted annealing.  Unfortunately, the obtained convergence bound depends on the dimension of the problem, and thus becomes uninformative in infinite dimensions. To address this issue, in Section \ref{sec:method} we carry out the convergence analysis directly in function spaces.

\section{Score-Based Diffusion Priors in Infinite Dimensions}\label{sec:score}
In \eqref{eq:PMC-RED-sun}, the score network $\textbf{S}_\theta(\textbf{x},\sigma)$ approximates $\nabla \log p_\sigma(\textbf{x})$. As mentioned above, $p_\sigma$ refers to the smoothed
prior, here being a distribution with respect to the Lebesgue measure. In infinite dimensions, however, there is no natural analogue of the Lebesgue measure; $p_\sigma$ is no longer well defined  \cite{da2006introduction}. To extend PMC-RED to infinite dimensions we then need to define the infinite-dimensional score function that will replace $\nabla \log p_\sigma(\textbf{x})$. Subsequently, we will show in Section \ref{sec:method} that this allows us to approximately sample from $\mu^{\textbf{y}}$ in the infinite-dimensional setting. 

Let $C_{\mu_0}:H\to H$ be a trace class, positive-definite, symmetric covariance operator. Here and throughout the paper, we assume that the prior $\mu_0$ that we want to learn from data to perform Bayesian inference in \eqref{def:IP} is the Gaussian measure 
\[\mu_0=\mathcal{N}(0,C_{\mu_0}),\]
though our analysis can be easily generalized to other classes of priors, e.g. priors given as a density with respect to a Gaussian.
To define the score function in infinite dimensions, we follow the approach outlined in \cite{baldassari2024conditional, pidstrigach2023infinite}. Let $C:H\to H$ be a trace class, positive-definite, symmetric covariance operator. Denote by $X_\tau$ the diffusion at time $\tau$ of a prior sample $X_0 \sim \mu_0$:
\[
X_\tau := e^{-\tau/2} X_0 + \int_0^\tau e^{-(\tau-s)/2} \sqrt{C} dW_s.
\]
$X_\tau$ evolves towards the Gaussian measure $\mathcal{N}(0,C)$ as $\tau \to \infty$ according to the SDE
\begin{equation}
dX_\tau = -\frac{1}{2}X_\tau d\tau + \sqrt{C} dW_\tau, \qquad X_0 \sim \mu_0.
\label{eq:fSDE}
\end{equation}
The score function in infinite dimensions is defined as follows:
\begin{definition}
The score function enabling the time-reversal of \eqref{eq:fSDE} is defined for $x \in H$ as
\begin{equation}
S(\tau,x; \mu_0):= -(1-e^{-\tau})^{-1} (x-e^{-\tau/2}\mathbb{E}[X_0|X_\tau=x]).
\label{def:score}
\end{equation}
\end{definition}
\begin{remark}\label{remark:score-network}
The neural network $S_\theta(\tau,x;\mu_0)$ that approximates the true score minimizes the denoising score matching loss in infinite dimensions:
\[
\mathbb{E}_{x_0 \sim \mathcal{L}(X_0), x_\tau \sim \mathcal{L}(X_\tau | X_0=x_0)} [\| S_\theta(\tau, x_\tau; \mu_0) - (1-e^{-\tau})^{-1} (x_\tau-e^{-\tau/2}x_0)\|_H^2 ],
\]
where $\mathcal{L}(X_0)$ and $\mathcal{L}(X_\tau|X_0=x_0)$ denote the law of $X_0$ and $X_\tau|X_0=x_0$, respectively.
\end{remark}
\begin{remark} 
$S_\theta(\tau,x;\mu_0)$ for some $\tau>0$ is what will replace the approximation of $\nabla \log p_\sigma$ in the finite-dimensional case \eqref{eq:PMC-RED-sun}.
\end{remark}
We will now make an assumption, often employed in infinite dimensions \cite{pidstrigach2023infinite, baldassari2024conditional},  on $C$ and $C_{\mu_0}$. 
\begin{assumption} \label{assumption-C-Cmu}
We assume that  $C_{\mu_0}, C:H\to H$ have the same basis of eigenvectors $(e_j)$ and that $C_{\mu_0} e_j= \mu_{0j} e_j$, $C e_j= \lambda_j e_j$ for every $j$, with $\lambda_j/\mu_{0j}<+\infty$.
\end{assumption}

\begin{remark}\label{remark:C-Cmu}
If $C=C_{\mu_0}$ or if $C$ is very close to $C_{\mu_0}$, in the sense that $C^{-1/2} C_{\mu_0} C^{-1/2}   -I $ is Hilbert-Schmidt \cite{da2014stochastic}, then $\lambda_j/\mu_{0j}<+\infty$ and Assumption \ref{assumption-C-Cmu} is fulfilled. 
\end{remark}

The proposition below has been proved by \citet{baldassari2024conditional}. We will reproduce the
proof for the reader’s convenience in Appendix \ref{sec:A}.

\begin{proposition}\label{prop:score-gaussian} 
Let Assumption \ref{assumption-C-Cmu} hold. Then 
\[
S(\tau,x; \mu_0)= - \sum_j  \left( \frac{e^\tau p_0^{(j)}}{1+(e^\tau-1)p_0^{(j)}}\right)x^{(j)} e_j = - C C_\tau^{-1} x, 
\]
where $x^{(j)}:=\langle x,e_j\rangle$, $p_0^{(j)}:=\frac{\lambda_j}{\mu_{0j}}$, and $C_\tau := e^{-\tau}C_{\mu_0} + (1-e^{-\tau})C$.
\end{proposition}

In Assumption \ref{assumption-C-Cmu}, $\lambda_j/\mu_{0j}<+\infty$ is needed to ensure the following result, which will be of help later.

\begin{corollary}\label{prop:score} By Proposition \ref{prop:score-gaussian}, 
\[
\begin{aligned}
C C_{\mu_0}^{-1}x+S(\tau,x; \mu_0) = (e^\tau-1) \sum_j^\infty \left(\frac{p_0^{(j)}- 1 }{1+(e^\tau-1)p_0^{(j)}}\right) p_0^{(j)} x^{(j)} e_j,
\end{aligned}
\]
meaning that the score mismatch error $\
\left\|S(\tau,x) + C C_{\mu_0}^{-1}x  \right\|^2_H$ goes to zero as $\tau$ as $\tau \to 0$.
\end{corollary}

\begin{remark}\label{remark:other-priors}
Similar results can be derived when $\mu_0$ is given as a density with respect to a Gaussian, $d\mu_0(x) = \textup{exp}(\Psi(x)) d\mathcal{N}(0,C_{\mu_0})$; see \cite[proof of Theorem 3]{pidstrigach2023infinite}.
\end{remark}

\section{Our Method: $\infty$-PMC-RED}\label{sec:method}
Inspired by \eqref{eq:PMC-RED-sun}, we utilize the score network $S_\theta (\tau, x; \mu_0)$ introduced in Remark \ref{remark:score-network} as a learning-based prior to define a new algorithm, called $\infty$-PMC-RED, that operates directly in function spaces. It is given by
\begin{equation}
X_{k+1} = X_k - \gamma \big(-C^{\alpha-1} S_\theta(\tau,X_k; \mu_0) -C^\alpha \nabla_{X_k} \log(\rho(\textbf{y}-\mathcal{A}(X_k))) \big) + (2 \gamma)^{\frac{1}{2}} Z_k,
\label{def:PMC-RED-infty}
\end{equation}
where $\gamma>0$ denotes the step-size, $\alpha>0$ is a constant that will be chosen later, and $Z_k = \int_{k}^{k+1} C^{\frac{\alpha}{2}} \text{d}W_t$ denotes the i.i.d Gaussian variables with mean zero and covariance operator $C^\alpha$.

Our main theoretical result, as summarized later in Theorem \ref{thm:1}, presents a convergence analysis of \eqref{def:PMC-RED-infty}, demonstrating that when $\tau$ is sufficiently small and $S_\theta$ provides a good approximation of the true score of the prior, it generates samples distributed approximately according to the true posterior. In subsection \ref{sec:main}, we provide a detailed analysis of this novel algorithm. As previously anticipated in the paper, $\infty$-PMC-RED is compatible with nonlinear inverse problems.

\subsection{Measure-Theoretic Definitions of the KL Divergence and the Fisher Information}

Before introducing our convergence theorem, we need to introduce analogues of the metrics appearing in \cite[Theorem 1]{sun2023provable}---namely, the Kullback-Leibler (KL) divergence and the relative Fisher information (FI)---that are compatible with the infinite-dimensional setting of our paper. 
Since, as we mentioned, there is no natural analogue of the Lebesgue measure in infinite-dimensional spaces, we will adopt a measure-theoretic definition of the KL divergence, as proposed by  \citet{ambrosio2005gradient}:
\[
\text{KL}(\nu||\mu)  := \int_H \log \frac{d\nu}{d\mu}(X) d\nu (X) 
\]
if $\nu \ll \mu$, where $d\nu/ d\mu$ refers to the Radon-Nikodym derivative; this quantity is set to infinity if $\nu$ is not absolutely continuous with respect to $\mu$. In our convergence theorem, we will employ 
\begin{equation}
\int_H \left\| C^{\frac{\alpha}{2}}\nabla_{X} \log \frac{d\nu}{d\mu}(X) \right\|_H^2 d\nu (X)
\label{def:distance}
\end{equation}
as a criterion to assess similarity between measures. As for the measure-theoretic KL divergence, we set \eqref{def:distance} to infinity if $\nu$ is not absolutely continuous with respect to $\mu$. It is straightforward to see that, if \eqref{def:distance} is zero,
then $\nu$ and $\mu$ are equal $\nu$-almost surely. 


\subsection{Theoretical Convergence Analysis}\label{sec:main} 
In this subsection, we aim to study the convergence of
\begin{equation}
\int_H \left\| C^{\frac{\alpha}{2}}\nabla_{X} \log \frac{d\nu_t}{d\mu^\textbf{y}}(X) \right\|_H^2 d\nu_t (X),
\label{eq:convergence-measure}
\end{equation}
where $\{\nu_t\}_{t\geq0}$ represents a continuous interpolation of the probability measures generated by \eqref{def:PMC-RED-infty}
\begin{equation}
\begin{aligned}
& X_{t} \! =\! X_{k\gamma} \! \! + \! (t-k\gamma) \! \left( C^{\alpha-1} S_\theta(\tau,X_{k\gamma}; \mu_0) \! + \! C^\alpha  \nabla_{X_{k\gamma}} \! \log (\rho(\textbf{y}-\mathcal{A}(X_{k\gamma})))\right)  \! +\! 2^\frac{1}{2}C^{\frac{\alpha}{2}} (W_t-W_{k\gamma})
\end{aligned}
\label{def:nu_t}  
\end{equation}
for $t \in [k\gamma, (k + 1)\gamma]$, with initial state $X_0 \sim \nu_0$, where $\gamma>0$ is the step-size, $S_\theta$ is a neural network approximating the score defined in \eqref{def:score}, and $\{W_t\}_{t\geq 0}$ is a Wiener process on $H$.

To prove the convergence of \eqref{eq:convergence-measure}, we will use the following assumptions. 

\begin{assumption}\label{assumption:phi} 
$\nabla_X \Phi_0 \in C^1$ is continuously differentiable and globally Lipschitz; for any $X_1,X_2 \in H$: \[\| \nabla \Phi_0 (X_1)- \nabla \Phi_0(X_2)\|_H 
\leq  L_{\Phi_0} \| X_1-X_2\|_H.\]
\end{assumption}

\begin{remark}\label{remark:phi}
Note that Assumption \ref{assumption:phi} does not assume the log-concavity of the likelihood, meaning that our analysis is compatible with nonlinear inverse problems. 
\end{remark}

\begin{assumption}\label{assumption:score}
For any $\tau>0$, the score network $S_\theta(\tau,X; \mu_0)$ approximating \eqref{def:score} is Lipschitz continuous with $L_\tau>0$ for any $X_1,X_2 \in H$:
\[\begin{aligned}
\|S_\theta(\tau,X_1; \mu_0) - S_\theta(\tau,X_2; \mu_0) \|_H & \leq L_\tau \|
X_1-X_2
\|_H
.
\end{aligned}
\]
Moreover, $S_\theta(\tau,X; \mu_0)$ has a bounded error $\epsilon_\tau<\infty$ for every $X \in H$:
\[
\|S_\theta(\tau,X; \mu_0) - S(\tau,X; \mu_0)\|_H \leq \epsilon_\tau.
\]
\end{assumption}

\begin{remark}\label{remark:lioschitz}
Note that the true score function \eqref{def:score} is Lipschitz continuous, as shown in Proposition \ref{prop:score-gaussian}. \end{remark}
\begin{remark}\label{remark:truncation}The assumption regarding the approximation error $\epsilon_\tau$ requires some consideration at $\tau=0$. This is because the score matching loss is unstable there; see \cite{kim2022soft}. In practice, it means that there is a trade-off between the score mismatch error, which goes to zero as $\tau \to 0$ (see Corollary \ref{prop:score}), and the approximation error term in Theorem \ref{thm:1}. 
\end{remark}


\begin{assumption}\label{assumption:factorisation}
The forward operator $\mathcal{A}$ depends only on $P^{D_0}(X)$ for some $D_0>0$. Moreover, we assume that the $\nu_0$ introduced in \eqref{def:nu_t} can be factorised as follows
\[
\nu_0(X) = \nu_0^{D_0}(X^{D_0})\prod_{j=D_0+1}^\infty \nu_0^{(j)}(X^{(j)}),
\]
where the superscript $D_0$ in $X^{D_0}$ refers to the orthogonal projection $P^{D_0}$ of the $H$-valued random variable $X$ onto the linear span of the first $D_0$ eigenvectors $(e_j)$ of $C$, $\nu_0^{D_0}:=P^{D_0}_{\#} \nu_0$, and $\nu_0^{(j)}$ is the density of $X^{(j)}: = \langle X,e_j\rangle$; see Appendix \ref{appendix:notation} for details on the notation. We also assume that
\[
\mu^{\textbf{y}}(X) = (\mu^{\textbf{y}})^{D_0}(X^{D_0}) \prod_{j=D_0+1}^\infty (\mu^{\textbf{y}})^{(j)}(X^{(j)}).
\]
\end{assumption}

\begin{remark}\label{remark:factorisation}
Assumption \ref{assumption:factorisation} implies that the algorithm does not explicitly depend on the articulation of the subspace associated with the first $D_0$ modes. Thus, the essential aspect of the assumption is that only a finite number of modes contributes to the observations, which is quite realistic from an applications point of view. Moreover, the error bound in Theorem \ref{thm:1} will not depend on $D_0$, ensuring the robustness of the convergence analysis of $\infty$-PMC-RED with respect to increasing $D_0$, which is crucial in an infinite-dimensional setting.      
\end{remark}

Now that we have listed the main assumptions for our convergence analysis, we are ready to state our main result, Theorem \ref{thm:1}. Specifically, we establish an explicit bound for \eqref{eq:convergence-measure}, which resembles that for PMC-RED in \cite[Theorem 1]{sun2023provable}, with the main difference being that ours does not diverge as the dimension of the problem goes to infinity.

\begin{theorem}[Convergence bound of $\infty$-PMC-RED] \label{thm:1}
Let Assumptions \ref{assumption-C-Cmu}--\ref{assumption:factorisation} hold. Denote by $\{\nu_t\}_{t\geq0}$ a continuous interpolation of the probability measures generated by \eqref{def:PMC-RED-infty}:
\[
\begin{aligned}
& X_{t} \! =\! X_{k\gamma} \! \! + \! (t-k\gamma) \! \left( C^{\alpha-1} S_\theta(\tau,X_{k\gamma}; \mu_0) \! + \! C^\alpha  \nabla_{X_{k\gamma}} \! \log (\rho(\textbf{y}-\mathcal{A}(X_{k\gamma})))\right)  \! +\! 2^\frac{1}{2}C^{\frac{\alpha}{2}} (W_t-W_{k\gamma})
\end{aligned}
\]
\text{for }$t \in [k\gamma, (k + 1)\gamma]$, \text{with initial state } $X_0 \sim \nu_0$, where $\gamma>0$ is the step-size, $S_\theta$ is a neural network approximating the score defined in \eqref{def:score}, and $\{W_t\}_{t\geq 0}$ is a Wiener process on $H$ independent of $X_t$. For $\gamma \in \left(0,\frac{1}{\sqrt{128}\textup{Tr}(C^{\alpha}) L}\right]$, we have \[
\begin{aligned}
 & \frac{1}{N\gamma} \int_{0}^{N\gamma} \left( \int \left\| C^{\frac{\alpha}{2}}\nabla_{X} \log \frac{d\nu_t}{d\mu^{\textbf{y}}}(X) \right\|_H^2 d\nu_t (X) \right) dt  \\ & \leq \frac{4 \textup{KL}(\nu_0||\mu^{\textbf{y}})}{N\gamma } + \big(\big(16 \sqrt{\textup{Tr}(C^\alpha)}+64\big) \textup{Tr}(C^\alpha) L^2 \big) \gamma + \underbrace{A_1}_{\substack{\text{Score Mismatch} \\ \text{Error}}} \tau^2 +\underbrace{A_2}_{\substack{\text{Approximation} \\ \text{Error}}}\epsilon^2_\tau,
\end{aligned}
\]
where $L=\max\left\{\sqrt{\textup{Tr}(C^{\alpha-2})L_\tau^2 + \textup{Tr}(C^\alpha)L^2_{\Phi_0}},L_{\Phi_0}\right\}$ and $N>0$ is the total number of iterations. Note that the constants $A_1$ and $A_2$ are independent of $\gamma$, $\tau$, and $\epsilon_\tau$. 
\end{theorem}

\begin{proof}[Proof. \textup{(Sketch)}\nopunct] We define the stochastic process
\[
X_t := X_0 + t \big( C^{\alpha-1} S(\tau,X_0; \mu_0)  +C^\alpha \nabla_{X_0} \log (\rho(\textbf{y} -\mathcal{A}(X_0)))\big) + 2^\frac{1}{2}C^{\frac{\alpha}{2}}  W_t, \quad X_0 \sim \nu_0.
\]
We derive the evolution equation for $\nu_t$ (the probability measure of $X_t$) and plug it into
the time derivative formula for KL($\nu_t||\mu^\textbf{y}$). 

We derive a bound relating $\int_H \left\| C^{\frac{\alpha}{2}}\nabla_X \log \frac{d\nu_t}{d\mu^{\textbf{y}}}(X) \right\|_H^2 d\nu_t (X)$  and the expected square $H$-norm $\mathbb{E}\big[\|
C^{\frac{\alpha}{2}-1} S(\tau,X; \mu_0)  +C^{\frac{\alpha}{2}} \nabla_{X} \log (\rho(\textbf{y} -\mathcal{A}(X)))
\|_H^2\big]$.

We construct a linear interpolation of \eqref{def:PMC-RED-infty}, make use of the aforementioned bounds and Assumptions \ref{assumption-C-Cmu}--\ref{assumption:factorisation}, and integrate the time derivative of the KL divergence over the interval $[k\gamma, (k+1)\gamma]$ to obtain a convergence bound that is dimension-free and depends explicitly on the score mismatch and the network approximation errors.

The full proof can be found in Appendix \ref{appendix:conv-thm}.
\end{proof}

\begin{corollary}[Stationarity of $\infty$-PMC-RED] \label{corollary:convergence}
Let $\alpha\geq3$. If $\gamma$, $\tau$, and $\epsilon_\tau$ are sufficiently small, then $\nu_t$ converges to $\mu^{\textbf{y}}$ in terms of \eqref{def:distance} at the rate of $\mathcal{O}\big( \frac{1}{N} \big)$.
\end{corollary}


\begin{remark}
\label{remark:thm1-2} 
A well-known heuristic to mitigate mode collapse and accelerate the sampling speed of Langevin MCMC algorithms is weighted annealing \cite{kirkpatrick1983optimization, neal2001annealed, song2019generative}. 
It consists of replacing $S_\theta(\tau,X_k;\mu_0)$ in \eqref{def:PMC-RED-infty} by $\eta_k S_\theta(\tau_k,X_k;\mu_0)$, where $(\eta_k)$ and $(\tau_k)$ decay from large initial values to $1$ and almost $0$, respectively. Following \cite[Proof of Theorem 3]{sun2023provable}, it is easy to show that weighted annealing will not introduce extra error in the convergence accuracy, even in the infinite-dimensional case.
\end{remark}

\section{Numerical Experiments}\label{sec:num}

We present two examples where we apply $\infty$-PMC-RED. In the first stylized example, we verify the accuracy of our method by applying it to a simple benchmark distribution. In the second example, we sample from the posterior of a nontrivial nonlinear PDE-based inverse problem, where the forward operator is the acoustic wave equation. We use $\infty$-PMC-RED to obtain samples from these distributions and adapt the preconditioning strategy proposed in \cite{LiEtAl_2016} to accelerate the convergence of the algorithm. We use \href{https://www.devitoproject.org/}{Devito} \citep{louboutin2019devito, luporini2020architecture} for the wave-equation-based simulations. All experiments were conducted on an Intel Xeon E5-2698 v4 CPU with 252GB of RAM.

\paragraph{Stylized example}
Following \citet{PaganiEtAl_2022}, we consider the 2D Rosenbrock distribution with density $p(x_1, x_2) \propto \exp ( - \frac{1}{2} x_{1}^{2} - (x_{2}-x_{1}^{2} )^{2})$, which can be directly sampled. It is chosen as the target distribution due to its shape: a curved, narrow ridge, as shown in Figure \ref{fig:stylized_true}, which usually poses a challenge for MCMC algorithms to explore \cite{PaganiEtAl_2022}.
We run $\infty$-PMC-RED with a starting step size of $\gamma=4.0$, which we decrease to $0.05$ over $10^5$ iterations (time of execution is about one minute) to ensure faster mixing of the chain. We designate the initial $10^3$ iterations as the burn-in phase, determined by the value of the unnormalized density during this period. The subsequent iterations, with a lag of $20$, serve as samples from the target distribution (5000 samples in total). Figure \ref{fig:stylized_estimated} shows the samples obtained from this process, which upon visual inspection indicate good agreement with the true samples of the Rosenbrock distribution.

\begin{figure}[t]
    \centering
    \begin{subfigure}[b]{0.4\textwidth}
        \centering
        \includegraphics[width=\textwidth]{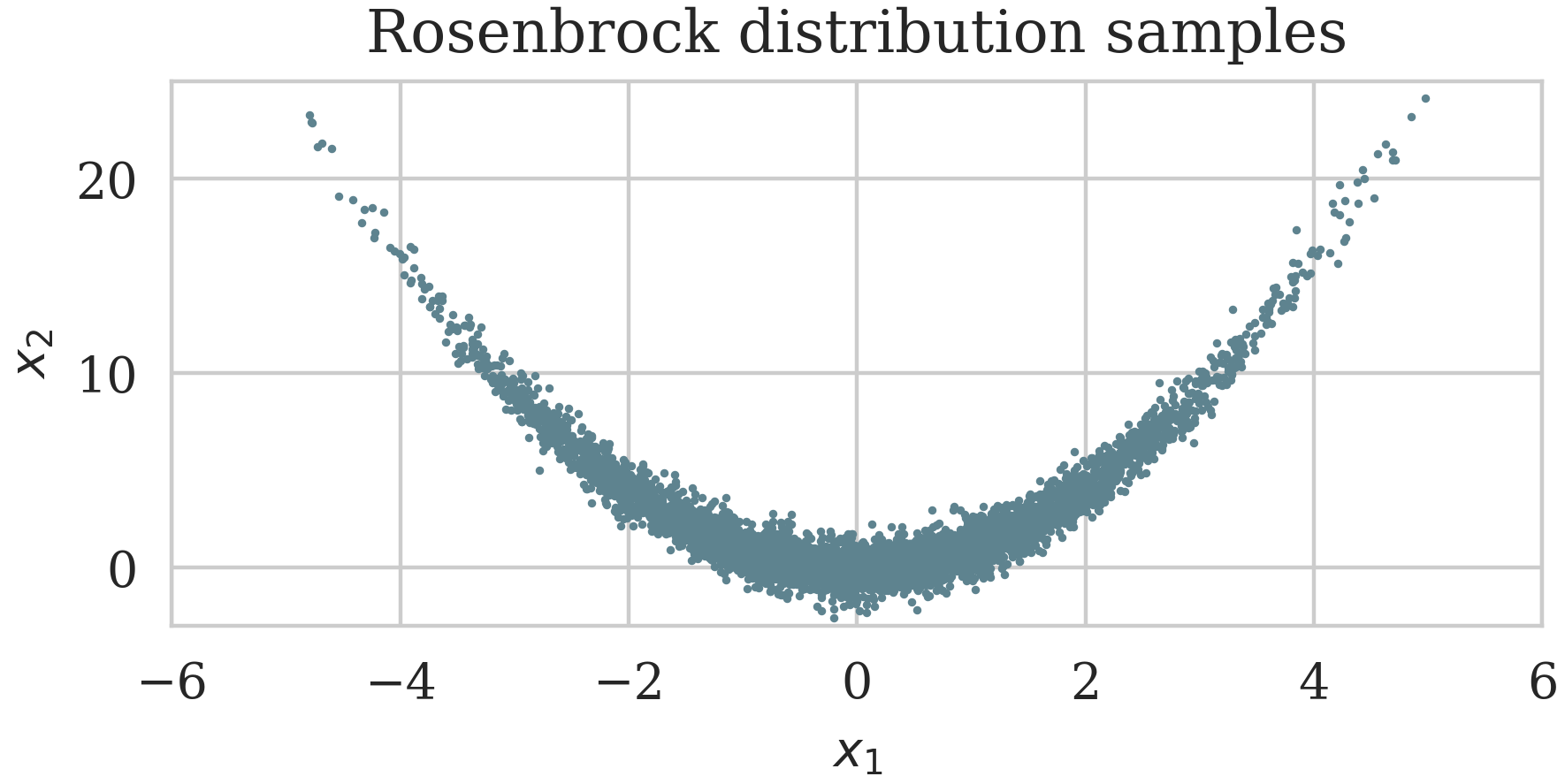}
        \caption{}
        \label{fig:stylized_true}
    \end{subfigure}
    \hspace*{2em}
    \begin{subfigure}[b]{0.4\textwidth}
        \centering
        \includegraphics[width=\textwidth]{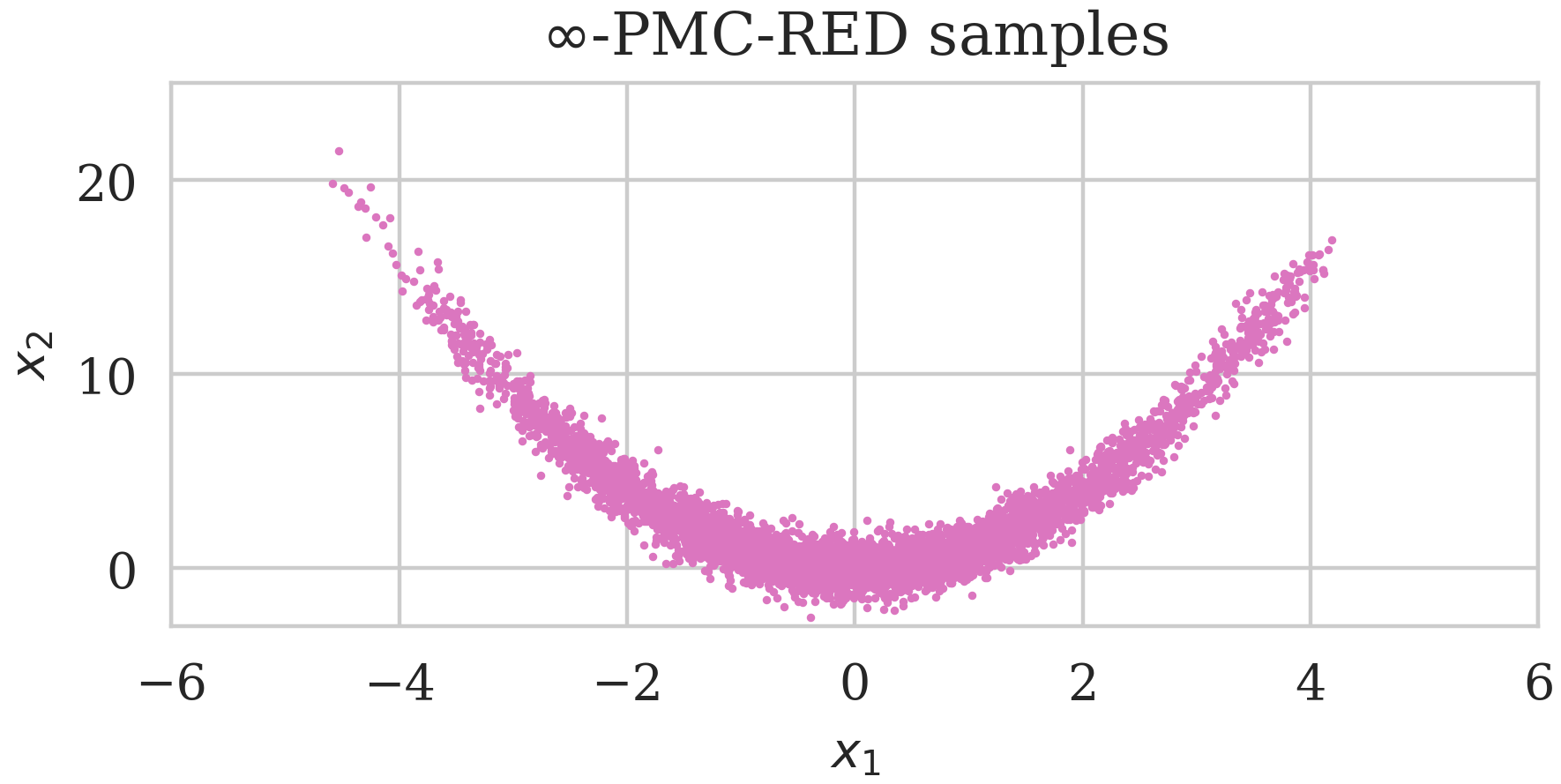}
        \caption{}
        \label{fig:stylized_estimated}
    \end{subfigure}

    \caption{Comparison of samples from (a) the true 2D Rosenbrock distribution and (b) samples obtained using $\infty$-PMC-RED.}
    \label{fig:toy_example}
\end{figure}

\paragraph{Nonlinear wave-equation based imaging example}

In this experiment, we address the problem of estimating the Earth's unknown subsurface squared-slowness model (see Figure \ref{fig:true_model}, which is a $250\PLH400$-dimensional 2D subset of the Compass model \cite{JonesEtAl_2012}) given measurements taken at the surface \cite{Tarantola_1984, VirieuxEtAl_2009}. These measurements are obtained by recording the seismic waves, generated by $101$ equispaced sources, through an array of $201$ equispaced receivers on the Earth's surface. Due to noise and the non-trivial nullspace of the forward operator, which arise from the presence of shadow zones and finite-aperture data \citep{lambare1992iterative, nemeth1999least}, prior information is required to obtain an accurate estimate of the squared-slowness model. Following \citet{PetersEtAl_2017}, we use the total variation norm, an edge-preserving prior commonly used in seismic imaging and image processing \cite{RudinEtAl_1992, PetersEtAl_2017}, to regularize the problem. Sampling involves running the $\infty$-PMC-RED algorithm with a starting step size of $\gamma=10^{-3}$, which we decrease to $\gamma=10^{-4}$ over $3 \times 10^4$ iterations according to the step size decay function introduced by \citet{TehEtAl_2016} (time of execution is about $22.5$ hours) to ensure faster mixing of the chain. During each iteration, and to reduce the number of PDE solves, akin to the stochastic gradient Langevin dynamics \cite{WellingEtAl_2011}, we obtain a stochastic approximation of the likelihood function by using one randomly selected source experiment from the total of 101 source experiments. We initialize the chain with an initial guess that is a heavily smoothed true model with a Gaussian kernel (see red lines in \ref{fig:vertical_profile}) and designate the initial $10^4$ iterations as the burn-in phase, determined by the value of the unnormalized density during this period. The subsequent iterations, with a lag of $10$, serve as samples from the target distribution ($2000$ samples in total). Figure \ref{figs:fwi} summarizes the results. Using the obtained posterior samples, we estimate the conditional mean (see Figure \ref{fig:conditional_mean}) and the pointwise standard deviation (see Figure \ref{fig:pointwise_std}) of the posterior distribution, using the latter to quantify the uncertainty. As expected, the regions of significant uncertainty correspond well with challenging-to-image sections of the model, qualitatively confirming the accuracy of our Bayesian inference method. This observation is more apparent in Figure \ref{fig:vertical_profile}, which shows two vertical profiles with $99\%$ confidence intervals (depicted as orange-colored shading), demonstrating the expected trend of increased uncertainty with depth. Furthermore, we notice that the ground truth (indicated by dashed black lines) mostly falls within the confidence intervals in most areas. We also observe a strong correlation between the pointwise standard deviation and the error in the conditional mean estimate (Figure~\ref{fig:model_error}), confirming the accuracy of our Bayesian inference method. 


\begin{figure}[!t]
    \centering
    \captionsetup[subfigure]{skip=-9pt}
    \hspace*{-1em}
    \begin{tabular}[t]{cc}
        \begin{tabular}{c}
        \begin{subfigure}[b]{0.35\textwidth}
            \includegraphics[width=\textwidth]{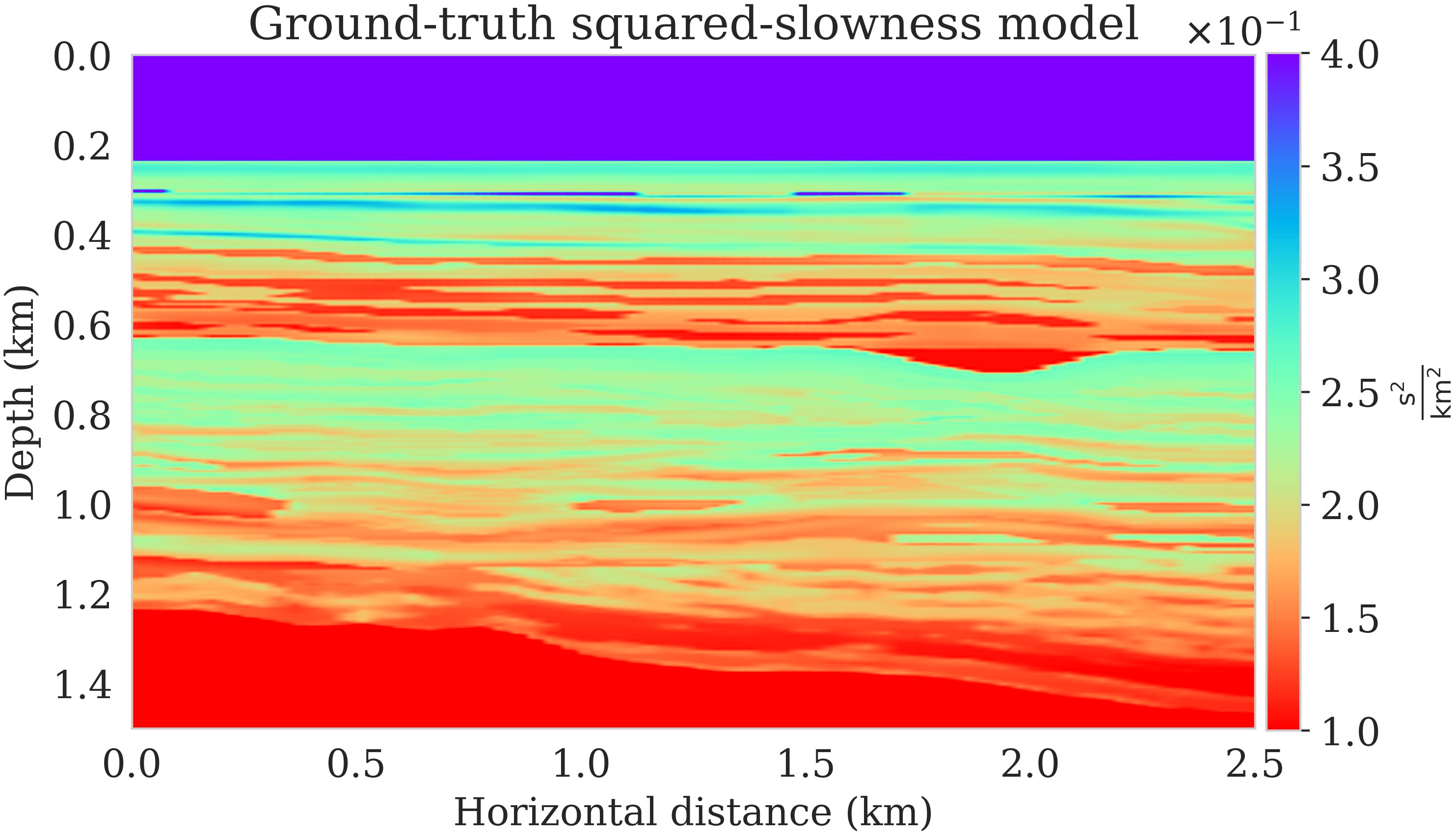}
        \vspace{0ex}\caption{}
        \label{fig:true_model}
        \end{subfigure}\hspace{0.5em}
        \begin{subfigure}[b]{0.35\textwidth}
            \includegraphics[width=\textwidth]{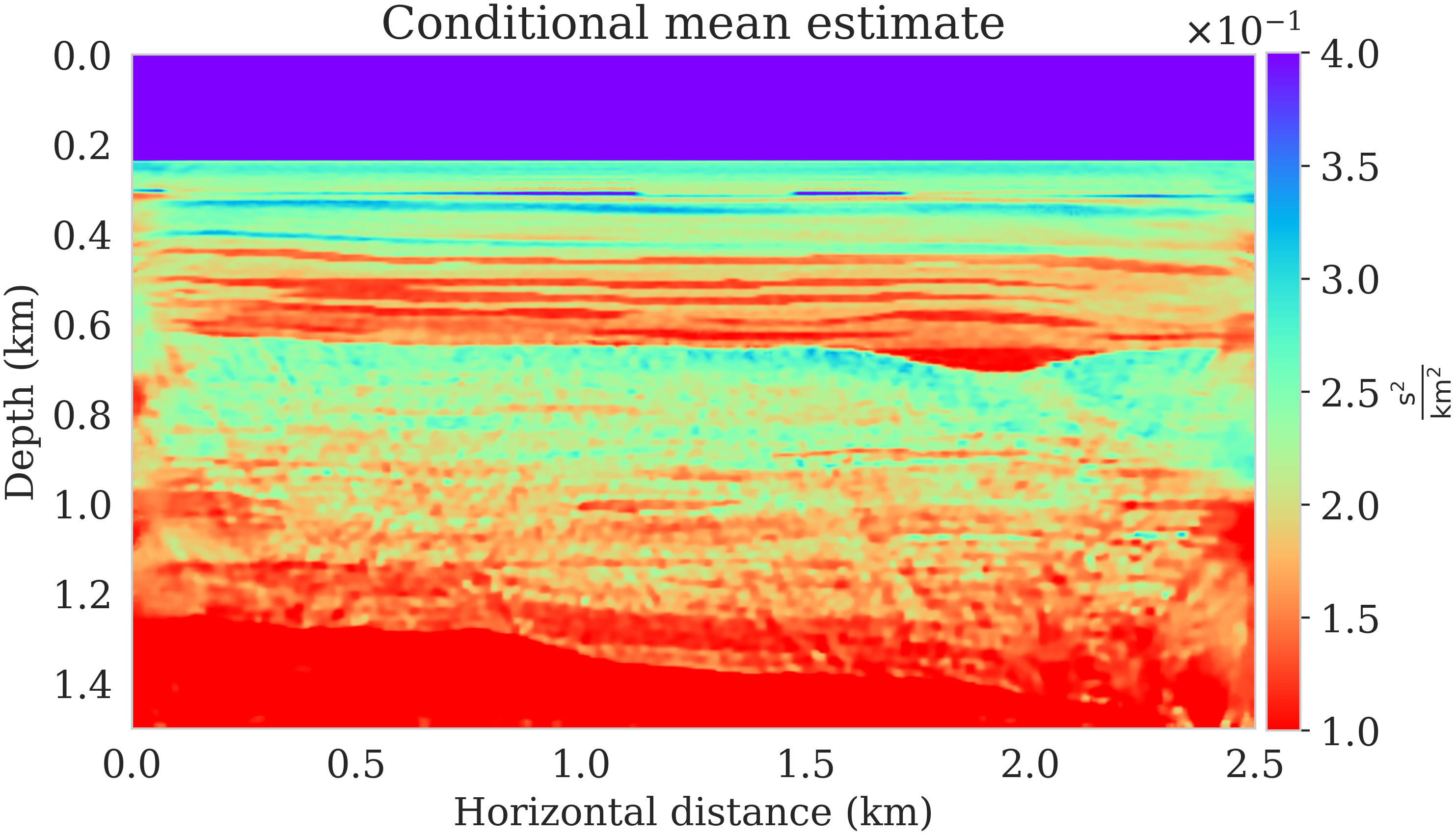}
        \vspace{0ex}\caption{}
        \label{fig:conditional_mean}
        \end{subfigure}\hspace{0em}
        \\
        \begin{subfigure}[b]{0.35\textwidth}
            \includegraphics[width=\textwidth]{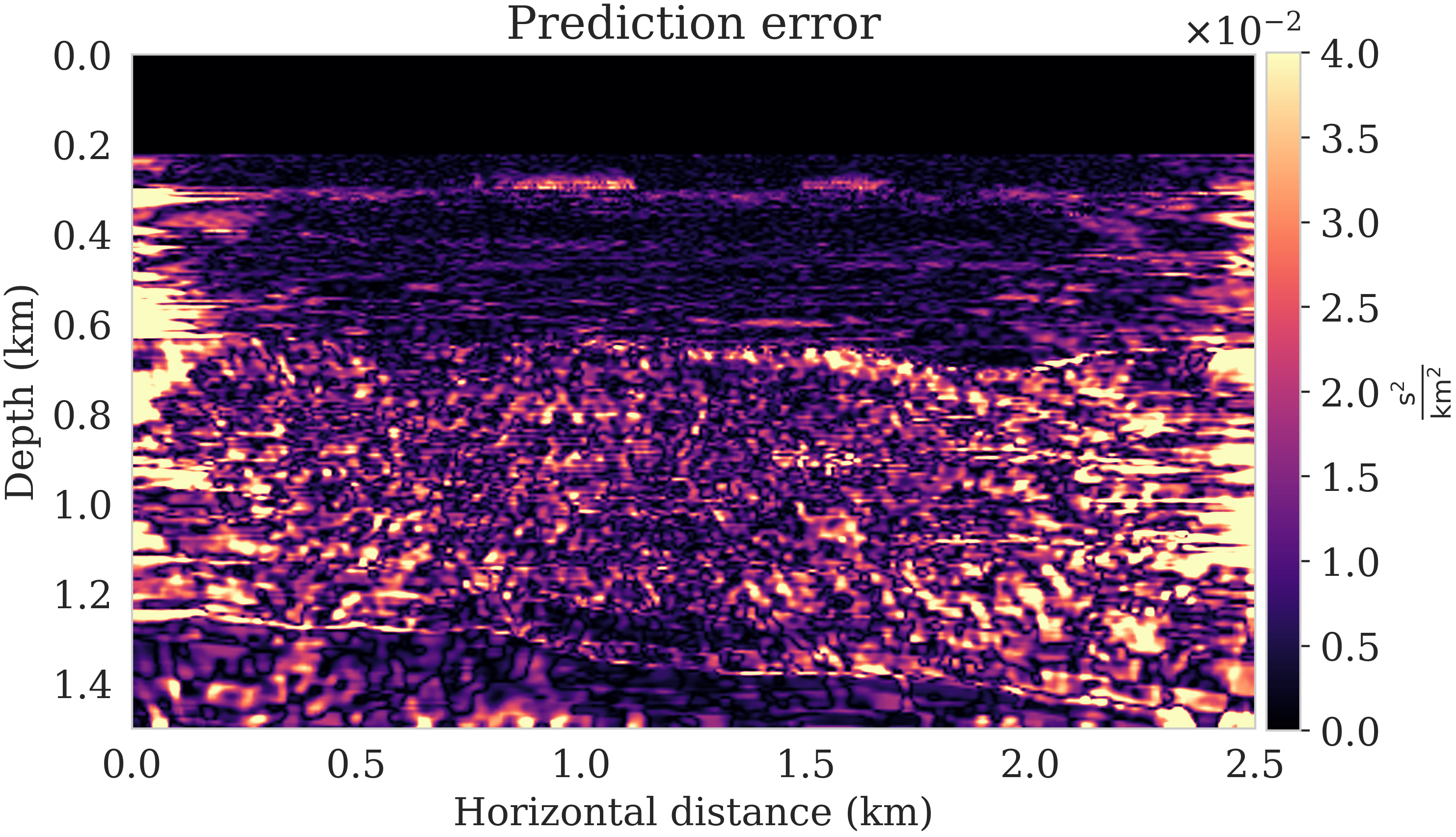}
        \vspace{0ex}\caption{}
        \label{fig:model_error}
        \end{subfigure}\hspace{0.85em}
        \begin{subfigure}[b]{0.35\textwidth}
            \includegraphics[width=\textwidth]{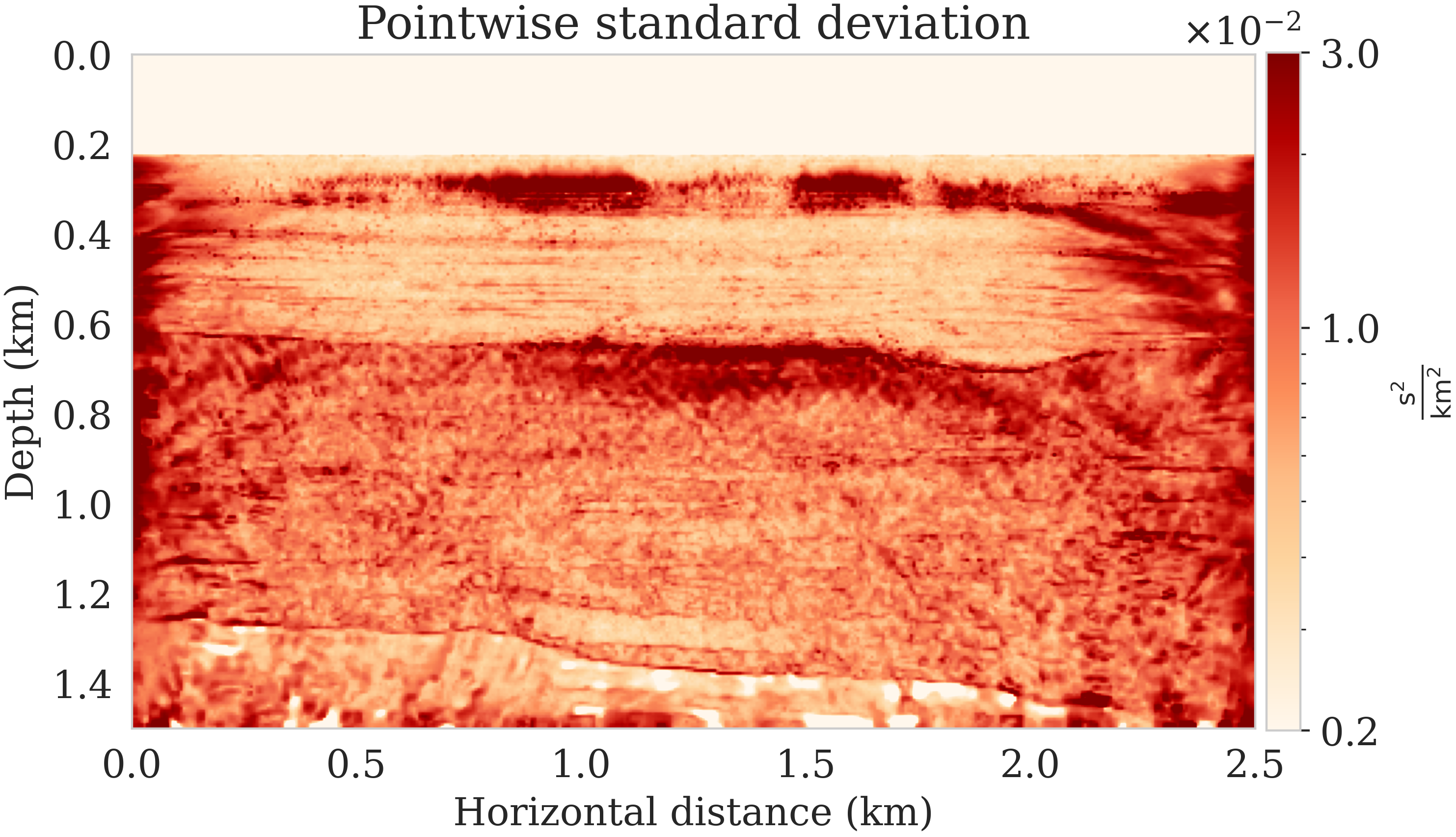}
        \vspace{0ex}\caption{}
        \label{fig:pointwise_std}
        \end{subfigure}\hspace{0em}
        \end{tabular}
    &
    \hspace{-2.5em}
    \begin{tabular}{c}
        \begin{subfigure}[b]{0.23\textwidth}
            \includegraphics[width=\textwidth]{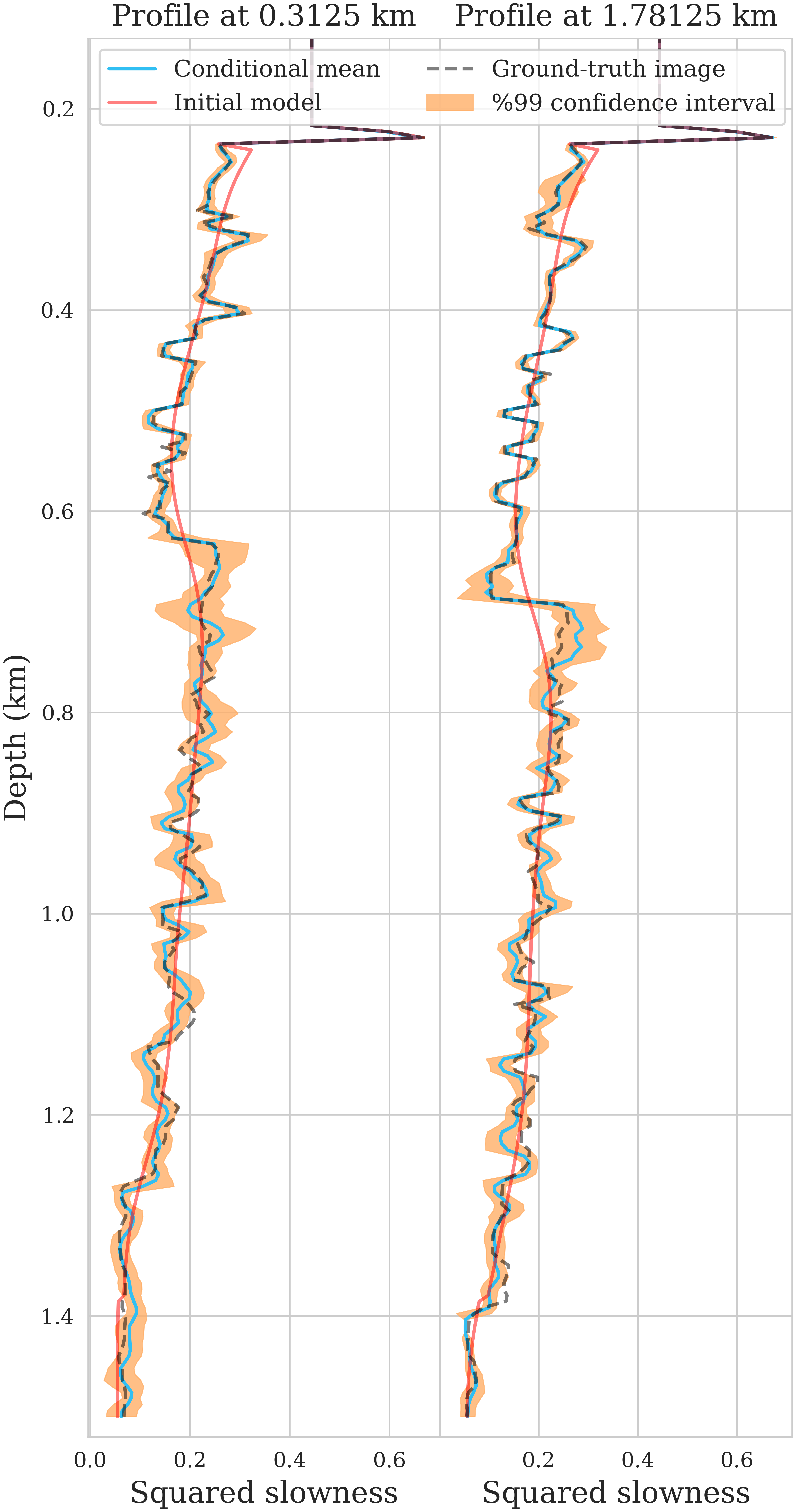}
        \vspace{0ex}\caption{}
        \label{fig:vertical_profile}
        \end{subfigure}\hspace{0em}
            \end{tabular}
    \end{tabular}
    \caption{Nonlinear wave-equation-based imaging and uncertainty quantification. (a) Ground-truth squared-slowness model. (b) Conditional (posterior) mean. (c) Absolute error between Figures \ref{fig:true_model} and \ref{fig:conditional_mean}. (d) Pointwise standard deviation. (e) Vertical profiles of the ground-truth squared-slowness mode, conditional mean estimate, and the $99\%$ confidence interval at two lateral positions.}
        \label{figs:fwi}
    \end{figure}

\section{Discussion and Conclusion}

We provided theoretical guarantees for a method using score-based diffusion priors within a Langevin-type MCMC scheme for functions. Extending the non-asymptotic convergence analysis of \citet{sun2023provable} to infinite dimensions, we proved that our algorithm converges to the true posterior, even with non-log-concave likelihoods and imperfect prior scores. We only require the log-likelihood to be continuously diﬀerentiable and globally Lipschitz, and we assume that only a finite number of modes contributes to the observations, which is quite realistic from an applications point of view. Crucially, the obtained convergence bound is dimension-free, ensuring the robustness of our analysis in the limit of infinite dimension. 

Despite the correctness of our convergence analysis, our method faces challenges, shared with similar works \cite{sun2023provable, feng2023score, xu2024provably}, regarding learning the score, which we did not address here. Indeed, to accurately sample the posterior distribution, the SDM must precisely estimate scores for both the initial point in the MCMC chain and all points during the burn-in phase. 
However, when $\tau$ is small, since there is no guarantee that the MCMC chain explores the high-probability regions of the prior during burn-in, the estimated scores might be inaccurate, possibly preventing the chain from converging to the true posterior. One possible heuristic for addressing this issue involves diffusing the samples from the prior at different decaying times, and estimating the corresponding scores by training a single score network,
as suggested by \cite{song2019generative}.
Another challenge comes from the non-convexity originating from the nonlinearity of the inverse problem, leading to computational challenges, especially in high dimensions. If not handled properly, $\infty$-PMC-RED may converge very slowly or, worse, get stuck in erroneous local minima. \citet{nickl2023bayesian} offers algorithmic guarantees that require strong assumptions on the forward model. Our work addresses the question of convergence with weaker assumptions and provides a heuristic, weighted annealing, for dealing with the issues of mixing time and local minima. We leave to future work the theoretical analysis of the computational complexity of $\infty$-PMC-RED, with the objective of formulating weaker assumptions than those proposed by \citet{nickl2023bayesian}.


\section*{Acknowledgments}
JG was supported by Agence de l’Innovation de D\'efense – AID - via Centre Interdisciplinaire d’Etudes pour la D\'efense et la S\'ecurité – CIEDS - (project 2021 - PRODIPO). AS and MVdH were supported by the Simons Foundation under the MATH + X program, the Department of Energy, BES under grant DE-SC0020345, and the corporate members of the Geo-Mathematical Imaging Group at Rice University. A significant part of the work of MVdH was carried out while he was an invited professor at the Centre Sciences des Donn\'{e}es at Ecole Normale Sup\'{e}rieure, Paris. KS was supported by Air Force Office of Scientific Research under grant FA9550-22-1-0176 and the National Science Foundation under grant DMS-2308389.
 
\bibliography{refs}


\appendix
 
\section{Proof of Proposition \ref{prop:score-gaussian}}\label{sec:A}

We define $X_\tau^{(j)} = \langle X_\tau,e_j\rangle $ and $S^{(j)}(\tau, x; \mu_0) = \langle S(\tau,x; \mu_0) , e_j \rangle $. We then have
\[
dX^{(j)}_\tau = -\frac{1}{2}X^{(j)}_\tau d\tau + \sqrt{\lambda_j} W_\tau^{(j)}.
\]
Since $C_{\mu_0}$ and $C$ have the same basis of eigenfunctions, the system of modes diagonalizes so that the $X^{(j)}_\tau$ processes are independent for different $j$ modes. Thus we have
\[
X_0^{(j)} = \sqrt{\mu_{0j}} \eta_0^{(j)}, \quad X_\tau^{(j)} = X_0^{(j)}e^{-\tau/2}+ \sqrt{\lambda_j(1-e^{-\tau})}\eta_1^{(j)}
\]
for $\eta_i^{(j)}$ independent standard Gaussian random variables. We seek
\[
x_0^{(j)} = \mathbb{E}[X_0^{(j)}|X_\tau^{(j)}=x^{(j)}],
\]
where $x_0^{(j)} = a x^{(j)}$ with $a$ solving
\[
\mathbb{E}[(aX_\tau^{(j)}-X_0^{(j)})X_\tau^{(j)}] =0,
\]
which gives
\[
a = \frac{e^{\tau/2}}{1+(e^\tau-1)p_0^{(j)}}
\]
for $$p_0^{(j)} = \frac{\lambda_j}{\mu_{0j}}.$$  Since also the time-reversed system diagonalizes, we have
\[S^{(j)}(\tau,x; \mu_0) = S^{(j)}(\tau,x^{(j)}; \mu_0) =  -\left( \frac{e^\tau p_0^{(j)}}{1+(e^\tau-1)p_0^{(j)}}\right)x^{(j)}. 
\]

\section{Proof of the Convergence Theorem}\label{appendix:conv-thm}

\subsection{Finite-Dimensional Projection}\label{appendix:notation}
Denote by $(e_j)$ the orthonormal basis of eigenvectors of a trace class, positive-definite, symmetric covariance operator $C$. 

\begin{definition}\label{def:span}
Define the linear span of the first $D$ eigenvectors as
\[
H^D := \left\{\sum_{j=1}^D f_j e_j|f_1,\ldots, f_D \in \mathbb{R} \right\} \subset H.
\]
Define $H^{D+1:\infty}$ such that $H= H^D \otimes H^{D+1:\infty}.$
\end{definition}

\begin{definition}\label{def:projection}
Let $P^D: H \to H^D$ be the orthogonal projection onto $H^D$. If we write an element $f$ of $H$ as
\[
f = \sum_{j=1}^\infty \langle f,e_j \rangle e_j,
\]
$P^D$ is equivalent to restricting $f$ to its first $D$ coefficients:
\[
P^D f = \sum_{j=1}^D \langle f,e_j \rangle e_j.
\]
\end{definition}
\begin{definition}\label{def:pushforward}
The push-forward $P^D_{\#}\mu$ of $\mu$ under $P^D$ is denoted by
\[
\mu^D :=P^D_{\#}\mu, \qquad \text{where} \quad P^D_{\#}\mu(A) = \mu((P^D)^{-1}(A)). 
\]
\end{definition}
\subsection{Useful Results on Measure Theory} 
Here we review some basic measure-theoretic tools needed in the proof of Theorem \ref{thm:1}.  
\begin{theorem}[Disintegration \cite{ambrosio2005gradient}]\label{thm:disintegration} Let $\mathcal{P}(H)$ be the family of all Borel probability measures on $H$. Let $H$,$Y$ be Radon separable metric spaces, $\mu \in \mathcal{P}(H)$, and $\pi: H \to Y$ a Borel map. Then there exists a $\pi_{\#}\mu$-a.e. uniquely determined Borel family of probability measures $\{\mu_y\}_{y\in Y} \subset \mathcal{P}(H)$ such that
\[
\mu_y(H\setminus \pi^{-1}(y)) = 0 \qquad for \; \pi_{\#}\mu\text{-}a.e. \; y \in Y
\]
and
\[
\int_H f(x) d\mu(x) = \int_Y \left( \int_{\pi^{-1}(y)} f(x) d \mu_y(x) \right) d\pi_{\#}\mu(y)  
\]
for every Borel map $f:H \to [0,\infty]$. In particular, when $H= H^D \times H^{D+1:\infty}$, $Y=H^D$, and $\pi=P^D$ (hence $\pi_{\#}\mu = \mu^D$), we can identify $\pi^{-1}(x^D)$ with $H^{D+1:\infty}$ and find a Borel family of probability measures $\{\mu_{x^D}\}_{x^D\in H^D} $ such that
\[
\mu_{x^D}(H^D)=0, \qquad \mu = \int_{H^D} \mu_{x^D} d\mu^D(x^D).
\]
\end{theorem}
The proof of the following result can be found in \cite[Corollary 9.4.6]{ambrosio2005gradient}. 
\begin{theorem}[KL divergence and orthogonal projection] \label{item:KL-projection}
For every measures $\nu, \mu$ on $H$, we have
\[
\lim_{D \to \infty} \textup{KL}(\nu^D||\mu^D) = \textup{KL}(\nu||\mu).
\]
\end{theorem}

\subsection{Lemmas}

Before going through the proof of Theorem \ref{thm:1}, we will need two lemmas. Similar results have been proved in \cite{sun2023provable, balasubramanian2022towards, vempala2019rapid} for the finite-dimensional setting.

\begin{lemma}\label{lem:1}
Let Assumptions \ref{assumption-C-Cmu} and \ref{assumption:factorisation} hold. Consider the stochastic process defined by
\[
X_t := X_0 - t Q_0 + 2^\frac{1}{2}C^{\frac{\alpha}{2}}  W_t, \quad with \quad Q_0 := Q_0(X_0), \; X_0 \sim \nu_0,
\]
where 
\[
Q_0(X_0) = -C^{\alpha-1} S(\tau,X_0; \mu_0)  -C^\alpha \nabla_{X_0} \log (\rho(\textbf{y} -\mathcal{A}(X_0))),
\]
and $\{W_t\}_{t\geq 0}$ is a Wiener process on $H$ independent of $X_0$. 
Then, writing $\nu_t$ for the probability measure of $X_t$, we have
\[
\begin{aligned}
\frac{d}{dt}\textup{KL} (\nu_t||\mu^{\textbf{y}}) \leq & - \frac{3}{4} \int \left\| C^{\frac{\alpha}{2}}  \nabla_{X} \log \frac{d\nu_t}{d\mu^{\textbf{y}} }(X) \right\|_H^2 d\nu_t (X)\\
& + \mathbb{E}_{\nu_t} \left[\|C^{\frac{\alpha}{2}} \nabla_X \Phi_0(X) +  C^{\frac{\alpha}{2}} C_{\mu_0}^{-1}   X - C^{-\frac{\alpha}{2}}   Q_0(X_0)\|_H^2\right].  
\end{aligned}
\]
\end{lemma}

\begin{proof}
We extend the proof of Lemma 1 in \citet{sun2023provable} to infinite dimensions. The main idea is to derive the evolution of the density of $\nu^D_t$, plug it into the
the time derivative formula for KL$(\nu^D_t||(\mu^\textbf{y})^D)$, and then take the limit as $D\to \infty$.

\paragraph{Step 1: Projecting the process onto a finite-dimensional subspace} Let $D>D_0$, where $D_0$ is defined in Assumption \ref{assumption:factorisation} as the number of modes contributing to the observations. Consider the stochastic process $X_t^D$ defined by
\begin{equation}
X_t^D := X^D_0 - t Q_0^D+ \sqrt{2P^DC^\alpha P^D} W_t, \quad X_0^D \sim \nu^D_0,
\label{eq:discretized-process}
\end{equation}
where
\[
\begin{aligned}
Q^D_0 :&= -(C^{\alpha-1})^DS^D(\tau,X_0; \mu_0) - (C^{\alpha})^D\nabla_{X_0^D} \log (\rho(\textbf{y} -\mathcal{A}(X_0)))\\& = -\sum_{j=1}^D \lambda_j^{\alpha-1}  S^{(j)}(\tau,X_0; \mu_0) e_j -\sum_{j=1}^D \lambda_j^\alpha \nabla_j \log (\rho(\textbf{y} -\mathcal{A}(X_0)))e_j, 
\end{aligned} 
\]
with 
\[
S^{(j)}(\tau,X_0; \mu_0):=\langle S(\tau,X_0; \mu_0),e_j\rangle, \qquad \nabla_j f:=\langle \nabla f,e_j\rangle.
\] We observe that 
\[
X_t^D= P^D (X_t).
\]
Since $X^D_t$ will stay in $H^D$ for all the times, we can view $X_t^D$ as a process on $\mathbb{R}^D$ and define the Lebesgue densities $\nu_t^D$ of $X_t^D$ there.  

\paragraph{Step 2: Deriving the evolution equation for $\nu_t^D$} For each $t>0$, let $\nu^{D}_{t,0}$ denote the Lebesgue density of the joint distribution of $(X_t^D,X^D_0)$. Let $\nu^D_{t|0}$ be the density of the conditional distribution of $X^D_t$ conditioned on $X^D_0$, and $\nu^D_{0|t}$ be the density of the probability distribution of $X^D_0$ conditioned on $X^D_t$. We have the relation 
\[
\nu_{t,0}^D(X^D,X^D_0) = \nu_{t|0}^D(X^D|X^D_0) \nu_0^D(X_0^D) = \nu_{0|t}(X_0^D|X^D)\nu_t^D(X^D).
\]
Since $S^D(\tau,X_0) = S^D(\tau,X_0^D)$ by Assumption \ref{assumption:factorisation}, conditioning on $X_0^D$ we have that $Q_0^D$ is a constant vector. Then, the conditional distribution $\nu_{t|0}^D$ evolves according to the following Fokker-Planck equation:
\[
\frac{\partial}{\partial t} \nu_{t|0}^D(X^D|X^D_0) = \text{div}_{X^D} \left( \nu_{t|0}^D(X^D|X_0^D) Q_0^D + (C^\alpha)^D \nabla_{X^D} \nu_{t|0}^D(X^D|X_0^D)\right).
\]

To derive the evolution equation for the marginal distribution $\nu_t^D(X^D)$, we need to take the expectation over $X_0^D \sim \nu_0^D$. Multiplying both sides of the Fokker-Planck equation by $\nu_0^D(X_0^D)$ and integrating over $X_0^D$, we have
\begin{equation}
\begin{aligned}
&\frac{\partial}{\partial t} \nu_t^D(X^D)  \\ & = \int \left( \frac{\partial}{\partial t} v_{t|0}^D(X^D|X_0^D)\right) \nu_0^D(X_0^D)dX_0^D\\
& = \int \text{div}_{X^D} \left( \nu_{t|0}^D(X^D|X_0^D)Q_0^D + (C^\alpha)^D \nabla_{X^D} \nu^D_{t|0} (X^D|X^D_0)\right) \nu_0^D(X_0^D) dX_0^D\\ 
& = \int \text{div}_{X^D} \left( \nu_{t,0}^D(X^D, X_0^D)Q_0^D + (C^\alpha)^D \nabla_{X^D} \nu^D_{t,0} (X^D,X^D_0)\right) dX_0^D \\
& =  \text{div}_{X^D} \! \left( \nu_t^D(X^D) \int \nu_{0|t}^D(X_0^D|X^D)Q_0^D dX_0^D + (C^\alpha)^D\nabla_{X^D} \int \nu^D_{t,0}(X^D,X_0^D)dX_0^D \right) \\
& = \text{div}_{X^D} \left( \nu_t^D(X^D) \mathbb{E}_{\nu^D_{0|t}} [Q_0^D|X_t^D=X^D] + (C^\alpha)^D\nabla_{X^D} \nu_t^D(X^D)\right).
\label{eq:evolution}
\end{aligned}
\end{equation}
\paragraph{Step 3: Calculating the derivative of the KL divergence} The time derivative of KL$(\nu^D_t||(\mu^{\textbf{y}})^D)$ is given by
\[
\begin{aligned}
&\frac{d}{dt}\text{KL} (\nu^D_t ||(\mu^{\textbf{y}})^D)  \\ & = \frac{d}{dt} \int \nu_t^D(X^D) \log \frac{\nu_t^D(X^D)}{(\mu^{\textbf{y}})^D(X^D)} dX^D \\
& = \int \frac{\partial \nu_t^D}{\partial t} \log  \frac{\nu_t^D(X^D)}{(\mu^{\textbf{y}})^D(X^D)} dX^D + \int \nu_t^D(X^D) \frac{\partial}{\partial t}\log \frac{\nu_t^D(X^D)}{(\mu^{\textbf{y}})^D(X^D)} dX^D \\
& = \int \frac{\partial \nu_t^D}{\partial t} \log  \frac{\nu_t^D(X^D)}{(\mu^{\textbf{y}})^D(X^D)} dX^D + \int \nu_t^D(X^D) \frac{(\mu^{\textbf{y}})^D(X^D)}{\nu_t^D(X^D)} \frac{1}{(\mu^{\textbf{y}})^D(X^D)} \frac{\partial \nu_t^D(X^D)}{\partial t} dX^D\\
& = \int \frac{\partial \nu_t^D}{\partial t} \log  \frac{\nu_t^D(X^D)}{(\mu^{\textbf{y}})^D(X^D)} dX^D + \frac{\partial}{\partial t}\int \nu_t^D(X^D) dX^D\\ 
& = \int \frac{\partial \nu_t^D}{\partial t} \log  \frac{\nu_t^D(X^D)}{(\mu^{\textbf{y}})^D(X^D)} dX^D . 
\end{aligned}
\]
By using the evolution equation for $\nu_t^D$ found in \eqref{eq:evolution}, we can derive
\[
\begin{aligned}
&\frac{d}{dt}\text{KL}(\nu_t^D||(\mu^{\textbf{y}})^D) \\
& = \int \frac{\partial \nu_t^D(X^D)}{\partial t} \log \frac{\nu_t^D(X^D)}{(\mu^{\textbf{y}})^D(X^D)} dX^D\\
& = \int \text{div}_{X^D} \left( \left( \nu_t^D(X^D) \mathbb{E}_{\nu^D_{0|t}} [Q_0^D|X_t^D=X^D] + (C^\alpha)^D\nabla_{X^D} \nu_t^D(X^D)\right) \right)\log \frac{\nu_t^D(X^D)}{(\mu^{\textbf{y}})^D(X^D)}  dX^D\\
& = - \int \left\langle \nu_t^D(X^D) \mathbb{E}_{\nu^D_{0|t}} [Q_0^D|X_t^D = X^D] + (C^\alpha)^D \nabla_{X^D} \nu_t^D (X^D), \nabla_{X^D}\log \frac{\nu_t^D(X^D)}{(\mu^{\textbf{y}})^D(X^D)}\right\rangle  dX^D\\
& = - \int \left\langle \nu_t^D(X^D) \left( \mathbb{E}_{\nu^D_{0|t}} [Q_0^D|X_t^D = X^D] + (C^\alpha)^D \nabla_{X^D}\log \frac{\nu_t^D(X^D)}{(\mu^{\textbf{y}})^D(X^D)} \right. \right. \\& \left. \left. \quad \qquad +(C^\alpha)^D \nabla_{X^D} \log (\mu^{\textbf{y}})^D(X^D) \right), \nabla_{X^D} \log \frac{\nu_t^D(X^D)}{(\mu^{\textbf{y}})^D(X^D)} \right\rangle dX^D \\
& = - \int \left\langle (C^\alpha)^D \nabla_{X^D} \log \frac{\nu_t^D (X^D)}{(\mu^{\textbf{y}})^D(X^D)}, \nabla_{X^D} \log \frac{\nu_t^D (X^D)}{(\mu^{\textbf{y}})^D(X^D)} \right\rangle \nu_t^D(X^D)  dX^D 
\\& \quad - \int  \left\langle  (C^\alpha)^D \nabla_{X^D} \log (\mu^{\textbf{y}})^D(X^D) +  \mathbb{E}_{\nu^D_{0|t}} [Q_0^D|X_t^D = X^D], \right.
\\& \left. \qquad \qquad \nabla_{X^D} \log \frac{\nu_t^D (X^D)}{(\mu^{\textbf{y}})^D(X^D)} \right\rangle \nu_t^D (X^D) dX^D\!.
\end{aligned}
\]
\paragraph{Step 4: Factorising $\nu_t$ and $\mu^{\textbf{y}}$ into product of marginals} As a consequence of Assumption \ref{assumption:factorisation} and the definition of $X_t$, $\nu_t$ can be factorised into two blocks $(1:D_0)$ and $(D_0+1:\infty)$. The latter can be further factorised into a product of marginals, since $S$ can be diagonalised and the likelihood does not depend on $P^{D_0+1:\infty}(X_0)$, once more by Assumption \ref{assumption:factorisation}. 
More precisely, we have 
    \[
    \nu_t(X) = \nu^{D_0}_t(X^{D_0}) \prod_{j=D_0+1}^{\infty} \nu_t^{(j)}(X^{(j)}), 
    \]
where $\nu_t^{D_0}$ is equivalent to $(\mu^\textbf{y})^{D_0}$ (they both have densities with respect to the Lebesgue measure over $\mathbb{R}^{D_0}$) and each $\nu_t^{(j)}$ is equivalent to $(\mu^\textbf{y})^{(j)}$. Then we have that
\[\frac{d \nu_t}{d\mu^{\textbf{y}}}(X) = \left( \frac{\nu_t^{D_0}}{(\mu^{\textbf{y}})^{D_0}}(X^{D_0})  \right) \prod_{j=D_0+1}^{\infty} \frac{\nu_t^{(j)}}{(\mu^{\textbf{y}})^{(j)}}(X^{(j)}). 
\]
In particular, for any $D > D_0$,
\[\frac{d\nu_t}{d\mu^{\textbf{y}}}(X) = \left( \frac{\nu_t^{D}}{(\mu^{\textbf{y}})^{D}}(X^{D})  \right) \prod_{j=D+1}^{\infty} \frac{\nu_t^{(j)}}{(\mu^{\textbf{y}})^{(j)}}(X^{(j)}),
\]
hence
\begin{equation}
\begin{aligned}
& C^\frac{\alpha}{2} \nabla_X \log \frac{d \nu_t}{d\mu^{\textbf{y}}}(X) \\ &=  (C^\frac{\alpha}{2})^D \nabla_{X^{D}}\log \left( \frac{d \nu_t^{D}}{d(\mu^{\textbf{y}})^{D}}(X^{D})  \right) 
+ (C^\frac{\alpha}{2})^{D+1:\infty}\nabla_{X^{D+1:\infty}}\log \left( \prod_{j=D+1}^{\infty} \frac{\nu_t^{(j)}}{(\mu^{\textbf{y}})^{(j)}}(X^{(j)})\right).
\label{eq:sum-full-factorization}
\end{aligned}
\end{equation}
\paragraph{Step 5: Taking the limit as $D\to\infty$} 
Assume that
\begin{equation}
\int_H \left\|C^\frac{\alpha}{2} \nabla_X \log \frac{d \nu_t}{d\mu^{\textbf{y}}}(X) \right\|_H^2 d\nu_t(X)<+\infty.
\label{eq:condition-1}
\end{equation}
By Theorem \ref{thm:disintegration}, disintegrating $
\nu_t$ with respect to $X^D$ yields
$$\begin{aligned}
 & \int_H \left\|C^\frac{\alpha}{2} \nabla_X \log \frac{d \nu_t}{d\mu^{\textbf{y}}}(X) \right\|_H^2 d\nu_t(X)  \\ & = 
\int_{H^D} \int_{H^{D+1:\infty}} 
\left\|C^\frac{\alpha}{2} \nabla_X \log \frac{d \nu_t}{d\mu^{\textbf{y}}}(X)\right\|_{H^D}^2 d (\nu_t)_{X^D} ( X^{D+1:\infty}) d\nu_t^D(X^D).
\end{aligned}
$$
We get
\begin{equation*}
\begin{aligned}
 & \int_H \left\|C^\frac{\alpha}{2} \nabla_X \log \frac{d \nu_t}{d\mu^{\textbf{y}}}(X) \right\|_H^2 d\nu_t(X)  \\ & = 
\int_{H^D} \int_{H^{D+1:\infty}} 
\left\|C^\frac{\alpha}{2} \nabla_X \log \frac{d \nu_t}{d\mu^{\textbf{y}}}(X)\right\|_{H}^2 \! \! d (\nu_t)_{X^D} ( X^{D+1:\infty}) d\nu_t^D(X^D)
\\ & = \int_{H^D} \int_{H^{D+1:\infty}} \left\| (C^\frac{\alpha}{2})^D \nabla_{X^{D}}\log \left( \frac{\nu_t^{D}}{(\mu^{\textbf{y}})^{D}}(X^{D})  \right) \right.
\\ & \qquad \qquad \left. + (C^\frac{\alpha}{2})^{D+1:\infty}\nabla_{X^{D+1:\infty}}\log \!\left( \prod_{j=D+1}^{\infty} \frac{\nu_t^{(j)}}{(\mu^{\textbf{y}})^{(j)}}(X^{(j)}) \right) \right\|^2_H \!\! \!\!  d(\nu_t)_{X^D}(X^{D+1:\infty}) d\nu_t^D(X^D)
\\&= \int_{H^D}
\left\|(C^\frac{\alpha}{2})^D \nabla_{X^{D}}\log \left( \frac{\nu_t^{D}}{(\mu^{\textbf{y}})^{D}}(X^{D})  \right)\right\|_{H^D}^2 d \nu_t^D(X^D)\\
& \qquad \qquad  + \int_{H^D} \int_{H^{D+1:\infty}}\sum_{j=D+1}^\infty \lambda^\alpha_j  \left|\nabla_j \log \left( \frac{\nu_t^{(j)}}{(\mu^{\textbf{y}})^{(j)}}(X^{(j)})\right) \right|^2  \!\! d (\nu_t)_{X^D} ( X^{D+1:\infty}) d\nu_t^D(X^D).
\end{aligned} 
\end{equation*}
By \eqref{eq:condition-1}, it follows that
$$
\int_{H^D} \int_{H^{D+1:\infty}} \sum_{j=D+1}^\infty \lambda_j^\alpha \left|\nabla_j \log \left( \frac{\nu_t^{(j)}}{(\mu^{\textbf{y}})^{(j)}}(X^{(j)})\right) \right|^2 d (\nu_t)_{X^D} ( X^{D+1:\infty}) \nu_t^D(X^D) dX^D
\stackrel{D \to +\infty}{\longrightarrow} 0.$$
This means that
\begin{equation}
\begin{aligned}
\lim_{D\to \infty} \int_{H^D} \left\| (C^\frac{\alpha}{2})^D \nabla_{X^D} \log \frac{d\nu_t^D}{d(\mu^{\textbf{y}})^D}(X^D) \right\|_H^2 d\nu^D_t(X^D) = \int_H \left\| C^{\frac{\alpha}{2}
} \nabla_X \log\frac{d\nu_t}{d\mu^{\textbf{y}}}(X)\right\|_H^2 d\nu_t(X).
\end{aligned}
\label{eq:usc}
\end{equation}
We now take the limit in
\[
\begin{aligned}
& \frac{d}{dt}\text{KL}(\nu_t^D||(\mu^{\textbf{y}})^D) \\
& = - \int \left\langle (C^\alpha)^D\nabla_{X^D} \log \frac{\nu_t^D (X^D)}{(\mu^{\textbf{y}})^D(X^D)}, \nabla_{X^D} \log \frac{\nu_t^D (X^D)}{(\mu^{\textbf{y}})^D(X^D)} \right\rangle \nu_t^D(X^D)  dX^D 
\\& \quad - \int  \left\langle (C^\alpha)^D \nabla_{X^D} \log (\mu^{\textbf{y}})^D(X^D) \!+\! \mathbb{E}_{\nu^D_{0|t}} [Q_0^D|X_t^D = X^D], \! \nabla_{X^D} \log \frac{\nu_t^D (X^D)}{(\mu^{\textbf{y}})^D(X^D)} \right\rangle \! d\nu_t^D (X^D)
\\& = - \int \left\langle (C^\alpha)^D \nabla_{X^D} \log \frac{\nu_t^D (X^D)}{(\mu^{\textbf{y}})^D(X^D)}, \nabla_{X^D} \log \frac{\nu_t^D (X^D)}{(\mu^{\textbf{y}})^D(X^D)} \right\rangle \nu_t^D(X^D)  dX^D 
\\& \quad -\int\! \! \left\langle \! \! (C^\frac{\alpha}{2})^D \nabla_{X^D} \log (\mu^{\textbf{y}})^D(X^D) +  (C^{-\frac{\alpha}{2}})^D\mathbb{E}_{\nu^D_{0|t}} [Q_0^D|X_t^D = X^D],\right. \\ & \quad \quad \quad \left. 
(C^\frac{\alpha}{2})^D \nabla_{X^D} \log \frac{\nu_t^D (X^D)}{(\mu^{\textbf{y}})^D(X^D)} \right\rangle \! d\nu_t^D (X^D) dX^D\!.
\end{aligned}
\]
By Theorem \ref{item:KL-projection}, we get
\[
\frac{d}{dt}\text{KL}(\nu_t^D||(\mu^{\textbf{y}})^D)\to \frac{d}{dt}\text{KL}(\nu_t||\mu^{\textbf{y}}) 
\]
for $D\to \infty$. By Eq. \eqref{eq:usc}, we have
\[
\begin{aligned}
& \lim_{D\to \infty} - \int \left\langle (C^\alpha)^D \nabla_{X^D} \log \frac{\nu_t^D (X^D)}{(\mu^{\textbf{y}})^D(X^D)}, \nabla_{X^D} \log \frac{\nu_t^D (X^D)}{(\mu^{\textbf{y}})^D(X^D)} \right\rangle \nu_t^D(X^D)  dX^D \\ & = - \int \left\| C^\frac{\alpha}{2}\nabla_{X} \log \frac{d\nu_t}{d\mu^{\textbf{y}} } (X)\right\|_H^2 d\nu_t (X).
\end{aligned}
\]
We apply Young's inequality
\[
\begin{aligned}
& -\int  \left\langle (C^\frac{\alpha}{2})^D \nabla_{X^D} \log (\mu^{\textbf{y}})^D(X^D) +(C^{-\frac{\alpha}{2}})^D \mathbb{E}_{\nu^D_{0|t}} [Q_0^D|X_t^D = X^D], \right. \\ & \quad \quad \quad \left. (C^\frac{\alpha}{2})^D \nabla_{X^D} \log \frac{\nu_t^D (X^D)}{(\mu^{\textbf{y}})^D(X^D)} \right\rangle \nu_t^D (X^D) dX^D \\
& \leq \frac{1}{4}\int \left\langle  (C^\alpha)^D \nabla_{X^D} \log \frac{\nu_t^D (X^D)}{(\mu^{\textbf{y}})^D(X^D)}, \nabla_{X^D} \log \frac{\nu_t^D (X^D)}{(\mu^{\textbf{y}})^D(X^D)} \right\rangle \nu_t^D(X^D)  dX^D \\
& \quad + \int \left\| -(C^\frac{\alpha}{2})^D \nabla_{X^D} \log (\mu^{\textbf{y}})^D(X^D) - (C^{-\frac{\alpha}{2}})^D  Q_0^D\right\|_H^2 \nu_t^D(X^D)  dX^D.
\end{aligned}
\]
We need to calculate
\[
\lim_{D\to \infty} \left\{ \int \left\| -(C^\frac{\alpha}{2})^D \nabla_{X^D} \log (\mu^{\textbf{y}})^D(X^D) -  (C^{-\frac{\alpha}{2}})^DQ_0^D \right\|_H^2 \nu_t^D(X^D)  dX^D\right\} ,
\]
as so far we have proved that
\[
\begin{aligned}
& \frac{d}{dt}\text{KL}(\nu_t||\mu^{\textbf{y}}) \\ \leq & - \frac{3}{4} \int \left\| C^\frac{\alpha}{2} \nabla_{X} \log \frac{d\nu_t (X )}{d\mu^{\textbf{y}}(X)} \right\|_H^2 d\nu_t (X) \\
& + \lim_{D\to \infty} \left\{ \int \left\|- (C^\frac{\alpha}{2})^D \nabla_{X^D} \log (\mu^{\textbf{y}})^D(X^D) -  (C^{-\frac{\alpha}{2}})^D Q_0^D \right\|_H^2 \! \nu_t^D(X^D)  dX^D\right\}\!.
\end{aligned}
\]
Recall that
\[
Q^D_0 =  - (C^{\alpha-1})^D S^D(\tau,X^D_0; \mu_0) -\sum_{j=1}^D \lambda_j^\alpha \nabla_j \log (\rho(\textbf{y} -\mathcal{A}(X_0)))e_j,
\]
where $S^D(\tau,X_0; \mu_0) = S^D(\tau,X_0^D; \mu_0)$ follows from the separability assumptions on $\nu_0$. We have
\[
\begin{aligned}
& \left\| -(C^\frac{\alpha}{2})^D \nabla_{X^D} \log (\mu^{\textbf{y}})^D(X^D) -  (C^{-\frac{\alpha}{2}})^DQ_0^D \right\|_H^2  \\
& = \! \left\| -(C^\frac{\alpha}{2})^D \nabla_{X^D} \! \log (\mu^{\textbf{y}})^D(X^D) \!+ \!(C^{\frac{\alpha}{2}-1})^D S^D(\tau,X_0^D; \mu_0)  \! +  \!\sum_{j=1}^D\! \lambda_j^{\frac{\alpha}{2}}\nabla_j \! \log (\rho(\textbf{y} \!-\!\mathcal{A}(X_0)))e_j\right\|_H^2\!\!\!\!.
\end{aligned}
\]
We would like to prove that
\[
\begin{aligned}
& \lim_{D\to \infty} \left\{ \int \left\| -(C^\frac{\alpha}{2})^D \nabla_{X^D} \log (\mu^{\textbf{y}})^D(X^D) + (C^{\frac{\alpha}{2}-1})^D S^D(\tau,X_0^D; \mu_0) \right. \right. \\ & \quad \quad \qquad 
\left. \left. +  \sum_{j=1}^D \lambda_j^{\frac{\alpha}{2}} \nabla_j \log (\rho(\textbf{y} -\mathcal{A}(X_0)))e_j\right\|_H^2 \nu_t^D(X^D)dX^D \right\} \\
& \leq \int \left\| -C^{\frac{\alpha}{2}} \nabla_{X} \Phi_0(X) - C^{\frac{\alpha}{2}} C_{\mu_0}^{-1} X + C^{\frac{\alpha}{2}-1} S(\tau,X_0; \mu_0)  \right. \\ & \qquad \left. +  C^{\frac{\alpha}{2}} \nabla_{X_0} \log (\rho(\textbf{y} - \mathcal{A}(X_0)))\right\|_H^2 d\nu_t (X ).
\end{aligned}
\]
First, notice that
\[
\begin{aligned}
& (C^\frac{\alpha}{2})^D \nabla_{X^D} \log (\mu^{\textbf{y}})^D(X^D) \\& = (C^\frac{\alpha}{2})^D \nabla_{X^D} \log \frac{(\mu^{\textbf{y}})^D(X^D)}{\mathcal{N}(0,C_{\mu_0}^D)(X^D)} + (C^\frac{\alpha}{2})^D \nabla_{X^D} \log \mathcal{N}(0,C_{\mu_0}^D) (X^D).
\end{aligned} 
\]
Then, since\[
(C^\frac{\alpha}{2})^D\nabla_{X^D} \log \mathcal{N}(0,C_{\mu_0}^D)(X^D) = -(C^\frac{\alpha}{2})^D(C_{\mu_0}^D)^{-1} X^D,
\]
we get
\[
\begin{aligned}
& \left\| -(C^\frac{\alpha}{2})^D\nabla_{X^D} \log (\mu^{\textbf{y}})^D(X^D) -  (C^{-\frac{\alpha}{2}})^D Q_0^D \right\|_H^2  \\
& = \left\| - (C^\frac{\alpha}{2})^D \nabla_{X^D} \log \frac{(\mu^{\textbf{y}})^D(X^D)}{\mathcal{N}(0,C_{\mu_0}^D)(X^D)}  +(C^\frac{\alpha}{2})^D(C_{\mu_0}^D)^{-1} X^D + (C^{\frac{\alpha}{2}-1})^D S^D(\tau,X_0^D; \mu_0) \right.\\
& \qquad \left.  +  \sum_{j=1}^D \lambda_j^\frac{\alpha}{2} \nabla_j \log (\rho(\textbf{y} -\mathcal{A}(X_0)))e_j \right\|_H^2.
\end{aligned}
\]
By Assumption \ref{assumption:factorisation} and the fact that the likelihood does not depend on $P^{D+1:\infty}(X_0)$ for any $D>D_0$, and since $d\mathcal{N}(0,C_{\mu_0}) = d\mathcal{N}(0,C_{\mu_0}^D) d\mathcal{N}(0,C_{\mu_0}^{D+1:\infty})$ (see \cite[Definition 1.5.2]{da2006introduction}), we can follow the same procedure that led to \eqref{eq:usc} and prove that 
\[
\begin{aligned}
& \lim_{D\to \infty} \left\{\int_{H^D} \left\| (C^\frac{\alpha}{2})^D \nabla_{X^D} \log (\mu^{\textbf{y}})^D(X^D) -  (C^{-\frac{\alpha}{2}})^D Q_0^D \right\|_H^2 \nu_t^D (X^D) dX^D \right\} \\
& = \int_H \| C^\frac{\alpha}{2} \nabla_{X} \Phi_0(X)  +  C^\frac{\alpha}{2} C_{\mu_0}^{-1}X + C^{\frac{\alpha}{2}-1} S (\tau,X_0; \mu_0 )\\ & \qquad + C^\frac{\alpha}{2} \nabla_{X_0}\log(\rho(\textbf{y}-\mathcal{A}(X_0))) \|^2_H d \nu_t(X),
\end{aligned}
\]
where we used  $d\mu^\textbf{y} \propto \text{exp}(-\Phi_0)d\mathcal{N}(0,C_{\mu_0})$  as per \eqref{def:posterior}. Putting everything together, we have
\[
\begin{aligned}
\frac{d}{dt}\text{KL}(\nu_t||\mu^{\textbf{y}}) \leq & - \frac{3}{4} \int \left\| C^\frac{\alpha}{2}\nabla_{X} \log \frac{d\nu_t}{d\mu^{\textbf{y}}}(X) \right\|_H^2 d\nu_t (X)\\&  + \mathbb{E}_{\nu_t} \left[\|C^\frac{\alpha}{2} \nabla_X \Phi_0(X) + C^\frac{\alpha}{2}C_{\mu_0}^{-1} X - C^{-\frac{\alpha}{2}} Q_0(X_0) \|^2_H\right] , 
\end{aligned}
\]
which ends the proof of the lemma.
\end{proof}

\begin{lemma}\label{lem:2}
Define $\mathcal{G}: H \to H$ as
\[
\mathcal{G}(X):= -C^{\alpha-1} S_\theta (\tau,X; \mu_0)  - C^\alpha \nabla_X \log(\rho(\textbf{y}-\mathcal{A}(X))), 
\]
where $S_\theta$ represents a neural network approximating the score defined in \eqref{def:score}. Let Assumptions \ref{assumption-C-Cmu}, \ref{assumption:phi}, and \ref{assumption:factorisation} hold. It holds that
\[
\begin{aligned}
\mathbb{E}_{\nu_t} [\| C^{-\frac{\alpha}{2}}\mathcal{G}(X)\|_H^2] \leq & 2\int \left\| C^\frac{\alpha}{2}\nabla_{X} \log \frac{d\nu_t}{d\mu^{\textbf{y}}}(X) \right\|_H^2 d\nu_t (X) + 4\textup{Tr}(C^\alpha) L_{\Phi_0} \\ & + 2 \mathbb{E}_{\nu_t} [\|C^{\frac{\alpha}{2}}\nabla \Phi_0(X)+ C^{\frac{\alpha}{2}} C_{\mu_0}^{-1}X-C^{-\frac{\alpha}{2}}\mathcal{G}(X)\|_H^2]. 
\end{aligned}
\]
\end{lemma}

\begin{proof} First, notice that 
\[
\begin{aligned}
\mathcal{G}(X) = \mathcal{G}(X) & -  (C^\alpha \nabla_X \Phi_0(X)+ C^\alpha C_{\mu_0}^{-1} X)  + C^\alpha \nabla_X \Phi_0(X) +  C^\alpha C_{\mu_0}^{-1}X.
\end{aligned}
\] 
We apply Young's inequality to get
\[
\begin{aligned}
& \mathbb{E}_{\nu_t}[\| C^{-
\frac{\alpha}{2}} \mathcal{G}(X) \|_H^2] \\ 
& \leq 2 \mathbb{E}_{\nu_t} [\| C^{
\frac{\alpha}{2}}\nabla_X \Phi_0(X) + C^{
\frac{\alpha}{2}} C_{\mu_0}^{-1}X\|_H^2] \\ & \quad + 2 \mathbb{E}_{\nu_t} [\| C^{
\frac{\alpha}{2}}\nabla_X\Phi_0(X) + C^{
\frac{\alpha}{2}}C_{\mu_0}^{-1}X - C^{-
\frac{\alpha}{2}}\mathcal{G}(X)\|_H^2].
\end{aligned}
\]
We study the first term of the inequality above. Notice that
\[
\begin{aligned}
& \mathbb{E}_{\nu_t} [\|  C^{
\frac{\alpha}{2}}\nabla_X \Phi_0(X) +  C^{
\frac{\alpha}{2}}C_{\mu_0}^{-1}X\|_H^2] = \mathbb{E}_{\nu_t }\left[\left\| - C^{
\frac{\alpha}{2}} \nabla_X \log \left(\frac{d\mu^{\textbf{y}}}{d\mu_0}\right)(X) +  C^{
\frac{\alpha}{2}}C_{\mu_0}^{-1}X \right\|_H^2 \right], \\
\end{aligned}
\]
where we used the relation 
\[
\nabla_X \Phi_0(X) = - \nabla_X \log\left(\frac{d\mu^{\textbf{y}}}{d\mu_0}\right)(X). 
\]
With the same arguments as in the proof of the previous lemma, we can write
\[
\begin{aligned}
& \mathbb{E}_{\nu_t} \left[\left\| - C^{
\frac{\alpha}{2}} \nabla_X \log \left(\frac{d\mu^{\textbf{y}}}{d\mu_0}\right)(X) +  C^{
\frac{\alpha}{2}}C_{\mu_0}^{-1}X \right\|_H^2 \right] \\ 
& = \lim_{D\to \infty} \mathbb{E}_{\nu^D_t} \left[\left\| -(C^\frac{\alpha}{2})^D \nabla_{X^D} \log \frac{(\mu^{\textbf{y}})^D(X^D)}{\mu_0^D(X^D)}  + (C^{
\frac{\alpha}{2}})^D (C^D_{\mu_0})^{-1} X^D \right\|^2_{H^D} \right]\\
& = \lim_{D\to \infty} \mathbb{E}_{\nu^D_t} \left[\left\| -(C^{
\frac{\alpha}{2}})^D\nabla_{X^D} \log  (\mu^{\textbf{y}})^D(X^D) + (C^{
\frac{\alpha}{2}})^D \nabla_{X^D}\log \mu_0^D(X^D)\right. \right. \\ & \qquad \qquad \qquad \left. \left. + (C^{
\frac{\alpha}{2}})^D (C^D_{\mu_0})^{-1}X^D \right\|^2_{H^D} \right]
\end{aligned}
\]
Since $\mu_0 = \mathcal{N}(0,C_{\mu_0})$, we have
\[
(C^{
\frac{\alpha}{2}})^D\nabla_{X^D}\log \mu_0^D(X^D) = -(C^{
\frac{\alpha}{2}})^D(C^D_{\mu_0})^{-1} X^D .
\]
It follows that
\[
\begin{aligned}
& \mathbb{E}_{\nu_t} \left[\left\| - C^{
\frac{\alpha}{2}} \nabla_X \log \left(\frac{d\mu^{\textbf{y}}}{d\mu_0}\right)(X) +  C^{
\frac{\alpha}{2}}C_{\mu_0}^{-1}X \right\|_H^2 \right] \\ & = \lim_{D\to \infty} \mathbb{E}_{\nu^D} \left[\left\| -(C^\frac{\alpha}{2})^D \nabla_{X^D} \log  (\mu^{\textbf{y}})^D(X^D)\right\|^2_{H^D} \right].
\end{aligned}
\]
We can derive
\[
\begin{aligned}
& \mathbb{E}_{\nu^D_t} \left[\left\| -(C^\frac{\alpha}{2})^D \nabla_{X^D} \log  ((\mu^{\textbf{y}})^D)(X^D)\right\|^2_{H^D}
\right]  \\
& =  \mathbb{E}_{\nu^D_t} \left[\left\| -(C^\frac{\alpha}{2})^D \nabla_{X^D} \log  (\mu^{\textbf{y}})^D(X^D) + (C^\frac{\alpha}{2})^D \nabla \log \nu_t^D(X^D) \right. \right. \\  & \qquad \qquad \left. \left. - (C^\frac{\alpha}{2})^D\nabla \log \nu_t^D(X^D) \right\|^2_{H^D} 
\right]\\
& =  \mathbb{E}_{\nu^D_t} \left[\left\| -(C^\frac{\alpha}{2})^D \nabla_{X^D} \log  (\mu^{\textbf{y}})^D(X^D) + (C^\frac{\alpha}{2})^D \nabla \log \nu_t^D(X^D)  \right\|^2_{H^D} \right. \\  & \qquad \qquad  \left. + 2 \langle (C^\frac{\alpha}{2})^D\nabla_{X^D} \log (\mu^{\textbf{y}})^D(X^D)  - (C^\frac{\alpha}{2})^D\nabla_{X^D} \log \nu_t^D(X^D),\right. \\  & \qquad \qquad  \left.(C^\frac{\alpha}{2})^D\nabla_{X^D} \log \nu_t^D(X^D) \rangle + \|(C^\frac{\alpha}{2})^D \nabla_{X^D} \log \nu_t^D(X^D) \|^2_{H^{D}} \right]\\
& = \mathbb{E}_{\nu_t^D} \left[\left\| -(C^\frac{\alpha}{2})^D \nabla_{X^D} \log  (\mu^{\textbf{y}})^D(X^D) + (C^\frac{\alpha}{2})^D\nabla \log \nu_t^D(X^D)  \right\|^2_{H^D} \right.\\
& \qquad \left. \quad + \langle 2(C^\frac{\alpha}{2})^D\nabla_{X^D} \log (\mu^{\textbf{y}})^D(X^D) - (C^\frac{\alpha}{2})^D \nabla_{X^D} \log \nu_t^D(X^D), \right. \\  & \qquad \qquad  \left.(C^\frac{\alpha}{2})^D\nabla_{X^D} \log \nu_t^D(X^D) \rangle\right]\\
& \leq \int_{H^D} \left\| (C^\frac{\alpha}{2})^D \nabla_{X^D} \log \frac{\nu_t^D(X^D)}{(\mu^{\textbf{y}})^D(X^D)}\right\|^2_{H^D} \nu_t^D(X^D) dX^D \\ 
& \qquad + 2 \mathbb{E}_{\nu_t^D} [\langle (C^\frac{\alpha}{2})^D \nabla_{X^D} \log (\mu^{\textbf{y}})^D(X^D), (C^\frac{\alpha}{2})^D \nabla_{X^D} \log \nu_t^D(X^D) \rangle],
\end{aligned}
\]
where in the last inequality we used the fact that $$-\mathbb{E}_{\nu^D} [\|(C^\frac{\alpha}{2})^D \nabla_{X^D}\log \nu_t^D(X^D)\|_H^2] \leq 0.$$ 
For Lemma \ref{lem:1}, we proved
\[
\begin{aligned}
& \lim_{D\to \infty} \int_{H^D} \left\| (C^\frac{\alpha}{2})^D \nabla_{X^D} \log \frac{\nu_t^D(X^D)}{(\mu^{\textbf{y}})^D(X^D)}\right\|^2_{H^D} \nu_t^D(X^D) dX^D
\\ & = \int_H \left\| C^\frac{\alpha}{2}\nabla_{X} \log \frac{d\nu_t}{d\mu^{\textbf{y}}}(X) \right\|_H^2 d\nu_t (X).
\end{aligned} 
\]
We notice that
\[
\begin{aligned}
& \mathbb{E}_{\nu_t^D} [\langle (C^\frac{\alpha}{2})^D \nabla_{X^D} \log (\mu^{\textbf{y}})^D(X^D), (C^\frac{\alpha}{2})^D\nabla_{X^D} \log \nu_t^D(X^D) \rangle] \\
& = \int_{H^D} \langle (C^\frac{\alpha}{2})^D \nabla_{X^D} \log (\mu^{\textbf{y}})^D(X^D), (C^\frac{\alpha}{2})^D\nabla_{X^D} \log \nu_t^D(X^D) \rangle \nu_t^D(X^D) dX^D\\
& = \int_{H^D} \langle (C^\frac{\alpha}{2})^D \nabla_{X^D} \log (\mu^{\textbf{y}})^D(X^D), (C^\frac{\alpha}{2})^D\nabla_{X^D} \nu_t^D(X^D) \rangle  dX^D\\
& = \int_{H^D} \langle (C^\alpha)^D\nabla_{X^D} \log (\mu^{\textbf{y}})^D(X^D), \nabla_{X^D} \nu_t^D(X^D) \rangle  dX^D\\
& = - \int_{H^D} \text{div}_{X^D}((C^\alpha)^D\nabla_{X^D} \log (\mu^{\textbf{y}})^D(X^D)) \nu_t^D(X^D) dX^D \leq \text{Tr}(C^\alpha) L_{\Phi_0} .
\end{aligned}
\]
In the second identity, we used $\nu_t^D(X^D) \nabla_{X^D} \log \nu_t^D(X^D) = \nabla_{X^D} \nu_t^D(X^D).$ In the fourth identity, we used 
\[
\begin{aligned}
& \text{div}_{X^D}((C^\alpha)^D\nabla_{X^D} \log (\mu^{\textbf{y}})^D(X^D) \nu_t^D(X^D)) \\
& = \text{div}_{X^D}((C^\alpha)^D \nabla_{X^D} \log (\mu^{\textbf{y}})^D(X^D)) \nu_t^D(X^D) + \langle (C^\alpha)^D \nabla_{X^D} \log (\mu^{\textbf{y}})^D(X^D),\nabla_{X^D}\nu_t^D(X^D)\rangle\end{aligned} 
\]
and 
\[
\int_{H^D} \text{div}_{X^D}\left((C^\alpha)^D\nabla_{X^D} \log (\mu^{\textbf{y}})^D(X^D)\nu_t^D(X^D)\right) dX^D =0. 
\]
In the last inequality, we used 
\[
\begin{aligned}
\nabla_{X^D} \log (\mu^{\textbf{y}})^D (X^D) & = \nabla_{X^D} \log \frac{(\mu^{\textbf{y}})^D (X^D)}{\mu_0^D(X^D)} + \nabla_{X^D} \log \mu_0^D (X^D) \\ & = - \nabla_{X^D} \tilde{\Phi}_0(X^D)   + (C^D_{\mu_0})^{-1}X^D,
\end{aligned}
\]
with
\[
\tilde{\Phi}_0(X^D) := \int_{H^{D+1:\infty}} \Phi_0(X^D, X^{D+1:\infty}) d(\mu^{\textbf{y}})_{X^D}(X^{D+1:\infty}).
\]
Then 
\[
\begin{aligned}
& - \int_{H^D}\text{div}_{X^D} ((C^\alpha)^D \nabla_{X^D} \log (\mu^{\textbf{y}})^D)(X^D) \nu_t^D(X^D) dX^D \\
& =  \int_{H^D}\text{div}_{X^D} ((C^\alpha)^D \nabla_{X^D} \tilde{\Phi}_0)(X^D) \nu_t^D(X^D) dX^D - \sum_{j=1}^D\left(\frac{\lambda_j^\alpha}{\mu_{0j}}\right)\int_{H^D}\nu_t^D(X^D) dX^D\\
& \leq  \text{Tr}(C^\alpha)  L_{\Phi_0} \int_{H^D}\nu_t^D(X^D) dX^D - \sum_{j=1}^D\left(\frac{\lambda_j^\alpha}{\mu_{0j}}\right)\\
& \leq \text{Tr}(C^\alpha)  L_{\Phi_0} .
\end{aligned}
\]
This is because, since we assumed that $\nabla \Phi_0 \in C^1$ is Lipschitz continuous with a constant $L_{\Phi_0}$, so is $\nabla_{X^D} \tilde{\Phi}_0(X^D)$,  and then 
\[
|\text{div}_{X^D} ((C^\alpha)^D \nabla_{X^D} \tilde{\Phi}_0 (X^D))| \leq \text{Tr}(C^\alpha) L_{\Phi_0}.
\] 
Combining all the inequalities above, we get
\[
\begin{aligned}
\mathbb{E}_{\nu_t} [\| C^{-\frac{\alpha}{2}}\mathcal{G}(X)\|_H^2] \leq & 2\int_H \left\| C^\frac{\alpha}{2}\nabla_{X} \log \frac{d\nu_t}{d\mu^{\textbf{y}}}(X) \right\|_H^2 d\nu_t (X) + 4\textup{Tr}(C^\alpha) L_{\Phi_0}\\ &   + 2 \mathbb{E}_{\nu_t} [\|C^{\frac{\alpha}{2}}\nabla \Phi_0(X) + C^{\frac{\alpha}{2}}C_{\mu_0}^{-1}X-C^{-\frac{\alpha}{2}}\mathcal{G}(X)\|_H^2]. 
\end{aligned}
\]
The proof is complete.
\end{proof}

\subsection{Proof of Theorem \ref{thm:1}}
We are now ready to prove our convergence theorem.

\begin{proof}[Proof of Theorem \ref{thm:1}] We construct the following interpolation for our method
\[
\begin{aligned}
& X_t = X_{k\gamma} - (t-k\gamma) \left( - C^{\alpha-1} S_\theta (\tau,X_{k\gamma}; \mu_0) - C^\alpha \nabla_{X_{k\gamma}} \log (\rho(\textbf{y}-\mathcal{A}(X_{k\gamma})))\right) \\ & \quad \quad \; + 2^\frac{1}{2}C^\frac{\alpha}{2}(W_t-W_{k\gamma}), 
\\[0.25em]
& \textup{for }t \in [k\gamma, (k + 1)\gamma].
\end{aligned}
\]
Let $\nu_t$ be the law of $X_t$. As in Lemma \ref{lem:2}, define
\[
\mathcal{G}(X,\tau) := -C^{\alpha-1}S_{\theta} (\tau, X; \mu_0) - C^\alpha \nabla_X \log(\rho(\textbf{y} - \mathcal{A}(X))). 
\]
As a consequence of Corollary \ref{prop:score}, the distance of $C^{-\frac{\alpha}{2}} \mathcal{G}(X,\tau)$ from $C^\frac{\alpha}{2}\nabla_X \Phi_0(X) +  C^\frac{\alpha}{2} C_{\mu_0}^{-1}X$ is given by
\begin{equation}
\begin{aligned}
&\| C^\frac{\alpha}{2} \nabla \Phi_0(X) + C^\frac{\alpha}{2}C_{\mu_0}^{-1} X - C^{-\frac{\alpha}{2}}\mathcal{G}(X,\tau)\|_H^2 \\ & \leq 2\| C^\frac{\alpha}{2} \nabla \Phi_0(X) + C^\frac{\alpha}{2}C_{\mu_0}^{-1} X + C^{\frac{\alpha}{2}-1}S(\tau,X; \mu_0) -C^\frac{\alpha}{2}\nabla \Phi_0(X)\|_H^2 \\& \qquad + 2\| C^{\frac{\alpha}{2}-1}(S_\theta(X,\tau; \mu_0) - S(X,\tau; \mu_0))\|_H^2 \\ & 
\text{(Plug-in Corollary \ref{prop:score})}
\\ 
& \leq 2\tau^2 K^2+2\| C^{\frac{\alpha}{2}-1}(S_\theta(X,\tau; \mu_0) - S(X,\tau; \mu_0))\|_H^2 
\\ & \text{(Assumption \ref{assumption:score})}
\\& \leq 2\tau^2 K^2 + 2\textup{Tr}(C^{{\alpha}-2}) \epsilon_\tau^2.
\end{aligned}
\label{eq:inequality1}
\end{equation}
Notice that
\begin{equation*}
\begin{aligned}
& \mathbb{E}_{\nu_t} [\| C^{-\frac{\alpha}{2}} \mathcal{G} (X_t,\tau) - C^{-\frac{\alpha}{2}} \mathcal{G}(X_{k\gamma},\tau)\|_H^2 \\ & \leq 2 \mathbb{E}_{\nu_t} [\| C^{\frac{\alpha}{2}-1} S_\theta (X_t,\tau; \mu_0) - C^{\frac{\alpha}{2}-1} S_\theta (X_{k\gamma},\tau; \mu_0)\|_H^2] \\ & \quad  + 2 \mathbb{E}_{\nu_t} [\| C^{\frac{\alpha}{2}} \nabla \Phi_0(X_t) - C^{\frac{\alpha}{2}} \nabla \Phi_0(X_{k\gamma})\|_H^2].
\end{aligned}\end{equation*}
By Assumptions \ref{assumption:phi} and \ref{assumption:score}, we have
\begin{equation}
\begin{aligned}
\mathbb{E}_{\nu_t} [\| C^{-\frac{\alpha}{2}}\mathcal{G} (X_t,\tau) - C^{-\frac{\alpha}{2}}\mathcal{G}(X_{k\gamma},\tau)\|_H^2] &\leq 2 \big(
\text{Tr}(C^{\alpha-2}) L^2_\tau  + \text{Tr}(C^{\alpha}) L_{\Phi_0}^2\big) \mathbb{E}_{\nu_t} [\| X_t - X_{k\gamma} \|_H^2]\\
& \leq 2L^2_{\mathcal{G}}\mathbb{E}_{\nu_t} [\| X_t - X_{k\gamma} \|_H^2],
\end{aligned}
\label{eq:inequality-G}
\end{equation}
where
\[
L_{\mathcal{G}}:=\sqrt{ \text{Tr}(C^{\alpha-2}) L^2_\tau  + \text{Tr}(C^{\alpha}) L_{\Phi_0}^2}.
\]
From Lemma \ref{lem:1}, we know for $t \in [k\gamma, (k+1)\gamma]$ that\begin{equation}
\begin{aligned}
\frac{d}{dt}\text{KL}(\nu_t||\mu^{\textbf{y}}) \leq & - \frac{3}{4} \int \left\| C^{\frac{\alpha}{2}}\nabla_{X_t} \log \frac{d\nu_t}{d\mu^{\textbf{y}}}(X_t) \right\|_H^2 d\nu_t (X_t)\\
& + \mathbb{E}_{\nu_t} \left[\|C^{\frac{\alpha}{2}} \nabla \Phi_0(X_t) +   C^{\frac{\alpha}{2}}C_{\mu_0}^{-1}X_t - C^{-\frac{\alpha}{2}} \mathcal{G}(X_{k\gamma},\tau)\|^2_H\right]. 
\end{aligned}
\label{eq:S12}
\end{equation}
The second term can be bounded via Young's inequality, \eqref{eq:inequality1} and \eqref{eq:inequality-G}:
\begin{equation}
\begin{aligned}
& \mathbb{E}_{\nu_t} \left[\|C^{\frac{\alpha}{2}} \nabla \Phi_0(X_t)  +  C^{\frac{\alpha}{2}}C_{\mu_0}^{-1}X_t - C^{-\frac{\alpha}{2}} \mathcal{G}(X_{k\gamma},\tau)\|^2_H\right] \\ 
& \leq 2 \mathbb{E}_{\nu_t} [\| C^{-\frac{\alpha}{2}} \mathcal{G} (X_t,\tau) - C^{-\frac{\alpha}{2}} \mathcal{G}(X_{k\gamma},\tau)\|_H^2] \\ & \quad + 2\mathbb{E}_{\nu_t}[ \| C^{\frac{\alpha}{2}} \nabla \Phi_0(X_t) +C^{\frac{\alpha}{2}}C_{\mu_0}^{-1}X_t - C^{-\frac{\alpha}{2}} \mathcal{G}( X_t,\tau)\|_H^2]\\
& \leq 4 L^2_{\mathcal{G}}\mathbb{E}_{\nu_t} [\| X_t - X_{k\gamma} \|_H^2] +  4(\tau^2 K^2 + \textup{Tr}(C^{\alpha-2}) \epsilon_\tau^2).
\end{aligned}
\label{eq:S13}
\end{equation}
We can bound the first term of the inequality above via
\[
\begin{aligned}
& \mathbb{E}_{\nu_t}[\|X_t-X_{k\gamma}\|_H^2] \\
& \leq 2 (t-k\gamma)^2 \text{Tr}(C^\alpha) \mathbb{E}_{\nu_t}\left[\left\| -C^{\frac{\alpha}{2}-1}  S_\theta (\tau,X_{k\gamma}; \mu_0)  - C^{\frac{\alpha}{2}} \nabla_{X_{k\gamma}} \log (\rho(\textbf{y}-\mathcal{A}(X_{k\gamma}))) \right\|_H^2\right]\\
& \quad + 4 \mathbb{E}_{\nu_t}[\|C^{\frac{\alpha}{2}} (W_t-W_{k\gamma})\|_H^2] \\
& \leq 2(t-k\gamma)^2 \text{Tr}(C^\alpha) \left( 2 \mathbb{E}_{\nu_t} \left[\left\|-C^{\frac{\alpha}{2}-1}S_\theta (\tau,X_{t}; \mu_0) - C^{\frac{\alpha}{2}} \nabla_{X_{t}} \log (\rho(\textbf{y}-\mathcal{A}(X_{t}))) \right\|_H^2\right] \right. \\&
\quad + \left. 4 L^2_{\mathcal{G}}\mathbb{E}_{\nu_t}[\| X_{k\gamma}-X_t \|_H^2] \right) + 4 \text{Tr}(C^\alpha)(t-k\gamma),
\end{aligned} \]
where for the last step we used Young's inequality and \eqref{eq:inequality-G}.
Rearranging the terms yields
\[
\begin{aligned}
& (1-8(t-k\gamma)^2 \text{Tr}(C^\alpha) L^2_{\mathcal{G}}) \mathbb{E}_{\nu_t}[\| X_{k\gamma}-X_t\|_H^2]\\
& \leq 4 (t-k\gamma)^2 \text{Tr}(C^\alpha)  \mathbb{E}_{\nu_t} \left[\left\|  -C^{\frac{\alpha}{2}-1}S_\theta (\tau,X_t; \mu_0) - C^{\frac{\alpha}{2}} \nabla_{X_t} \log (\rho(\textbf{y}-\mathcal{A}(X_t))) \right\|_H^2\right] \\ &\quad  + 4 \text{Tr}(C^\alpha )(t-k\gamma),
\end{aligned}
\]
which can be simplified by letting $\gamma \leq \frac{1}{4 \sqrt{\text{Tr}(C^\alpha)}L_{\mathcal{G}}} \Rightarrow 1-8(t-k\gamma)^2\text{Tr}(C^\alpha)L^2_{\mathcal{G}}  \geq 1-8\gamma^2 \text{Tr}(C^\alpha)L^2_{\mathcal{G}}  \geq \frac{1}{2}$. 
Therefore, when $\gamma \leq \frac{1}{4 \sqrt{\text{Tr}(C^\alpha)}L_{\mathcal{G}}}$, it holds that
\begin{equation}
\begin{aligned}
& \mathbb{E}_{\nu_t}[\| X_{k\gamma}-X_t\|_H^2]\\
& \leq 8 (t-k\gamma)^2 \text{Tr}(C^\alpha) \mathbb{E}_{\nu_t} \left[\left\|-C^{\frac{\alpha}{2}-1}S_\theta (\tau,X_t; \mu_0) - C^{\frac{\alpha}{2}} \nabla_{X_t} \log (\rho(\textbf{y}-\mathcal{A}(X_t))) \right\|_H^2\right] \\ & \quad + 8 \text{Tr}(C^\alpha)(t-k\gamma).
\end{aligned}
\label{eq:S14}
\end{equation}
By plugging \eqref{eq:S14} and \eqref{eq:S13} into \eqref{eq:S12} and invoking Lemma \ref{lem:2}, we can obtain
\begin{equation}
\begin{aligned}
&\frac{d}{dt}\text{KL}(\nu_t || \mu^{\textbf{y}}) \\
& \text{(Plug-in Eq. \eqref{eq:S12} and Eq. \eqref{eq:S13})}\\
& \leq -\frac{3}{4} \int \left\| C^{\frac{\alpha}{2}}\nabla_{X_t} \log \frac{d\nu_t}{d\mu^{\textbf{y}}}(X_t) \right\|_H^2 \!\!\! d\nu_t (X_t)  \! +  \! 4 L^2_{\mathcal{G}}\mathbb{E}_{\nu_t} [\| X_t - X_{k\gamma} \|_H^2] \! + \!  4(\tau^2 K^2 \! + \!\textup{Tr}(C^{\alpha-2}) \epsilon_\tau^2) \\
&\text{(Plug-in Eq. \eqref{eq:S14})}\\
& \leq -\frac{3}{4} \int \left\| C^{\frac{\alpha}{2}}\nabla_{X_t} \log \frac{d\nu_t}{d\mu^{\textbf{y}}}(X_t) \right\|_H^2 d\nu_t (X_t)
\\& \quad + 32 (t-k\gamma)^2 \textup{Tr}(C^\alpha) L^2_{\mathcal{G}}  \mathbb{E}_{\nu_t} \left[\left\|  -C^{\frac{\alpha}{2}-1}S_\theta (\tau,X_{t}; \mu_0) - C^{\frac{\alpha}{2}} \nabla_{X_{t}} \log (\rho(\textbf{y}-\mathcal{A}(X_{t}))) \right\|_H^2\right] \\& \quad + 32 \text{Tr}(C^\alpha) (t-k\gamma) L^2_{\mathcal{G}} + 4(\tau^2 K^2 + \textup{Tr}(C^{\alpha-2}) \epsilon_\tau^2)\\
&\text{(Plug-in Lemma \ref{lem:2} with }\|C^{\frac{\alpha}{2}} (\nabla \Phi_0(X) +  C_{\mu_0}^{-1} X) - C^{-\frac{\alpha}{2}} \mathcal{G}(X)\|_H^2 \leq 2(\tau^2 K^2 + \textup{Tr}(C^{\alpha-2}) \epsilon_\tau^2)))\\
& \leq -\frac{3}{4}\int \left\| C^{\frac{\alpha}{2}}\nabla_{X_t} \log \frac{d\nu_t}{d\mu^{\textbf{y}}}(X_t) \right\|_H^2 d\nu_t (X_t) \\
& \quad + 64 (t-k\gamma)^2 \textup{Tr}(C^\alpha) L^2_{\mathcal{G}} \! \left( \int \left\| C^{\frac{\alpha}{2}}\nabla_{X_t} \log \frac{d\nu_t}{d\mu^{\textbf{y}}}(X_t) \right\|_H^2 \! \! d\nu_t (X_t) \! + \! 2  \textup{Tr}(C^\alpha) L_{\Phi_0} \right. \\ & \left. \quad  + 2(\tau^2 K^2 +\textup{Tr}(C^{\alpha-2}) \epsilon_\tau^2) \! \right) + 32 \text{Tr}(C^\alpha) (t-k\gamma) L^2_{\mathcal{G}} + 4 (\tau^2 K^2 + \textup{Tr}(C^{\alpha-2}) \epsilon_\tau^2).
\end{aligned}
\label{eq:S15}
\end{equation}
Let $L=\max\{L_{\mathcal{G}}, L_{\Phi_0} \}$. We  can simplify \eqref{eq:S15} by letting $\gamma \leq \frac{1}{\sqrt{128 \text{Tr}(C^\alpha)}L}\Rightarrow 64(t-k\gamma)^2 \text{Tr}(C^\alpha) L^2 \leq 64\gamma^2 \text{Tr}(C^\alpha) L^2 \leq 
\frac{1}{2}$. Therefore, once $\gamma \leq \frac{1}{\sqrt{128 \text{Tr}(C^\alpha)}L}$, we get
\begin{equation}
\begin{aligned}
& \frac{d}{dt}\text{KL}(\nu_t||\mu^{\textbf{y}})\\
& \leq -\frac{1}{4} \int \left\| C^{\frac{\alpha}{2}}\nabla_{X_t} \log \frac{d\nu_t}{d\mu^{\textbf{y}}}(X_t) \right\|_H^2 d\nu_t (X_t)
\\& \quad + 64(t-k\gamma)^2 \text{Tr}(C^\alpha) L^2 \big(2\text{Tr}(C^\alpha) L + 2(\tau^2 K^2 + \textup{Tr}(C^{\alpha-2}) \epsilon_\tau^2)\big) \\& \quad + 32 (t-k\gamma) \text{Tr}(C^\alpha) L^2 + 4(\tau^2 K^2 + \textup{Tr}(C^{\alpha-2}) \epsilon_\tau^2).
\end{aligned}
\label{eq:S16}
\end{equation}
By integrating \eqref{eq:S16} between $[k\gamma, (k+1)\gamma]$ we get
\[
\begin{aligned}
&\text{KL}(\nu_{(k+1)\gamma}||\mu^{\textbf{y}}) - \text{KL}(\nu_{k\gamma}||\mu^{\textbf{y}})\\
& \leq -\frac{1}{4} \int_{k\gamma}^{(k+1)\gamma} \left( \int \left\| C^{\frac{\alpha}{2}}\nabla_{X_t} \log \frac{d\nu_t}{d\mu^{\textbf{y}}}(X_t) \right\|_H^2 d\nu_t (X_t) \right) dt \\
& \quad + \frac{64}{3} \text{Tr}(C^\alpha) L^2 \gamma^3 \big(2\text{Tr}(C^\alpha) L + 2(\tau^2 K^2 + \textup{Tr}(C^{\alpha-2}) \epsilon_\tau^2)\big) \! + \!16\gamma^2 \text{Tr}(C^\alpha) L^2 \\ & \quad + 4\gamma (\tau^2 K^2 + \textup{Tr}(C^{\alpha-2}) \epsilon_\tau^2)\\
&= -\frac{1}{4} \int_{k\gamma}^{(k+1)\gamma} \left( \int \left\| C^{\frac{\alpha}{2}}\nabla_{X_t} \log \frac{d\nu_t}{d\mu^{\textbf{y}}}(X_t) \right\|_H^2 d\nu_t (X_t) \right) dt \\
& \quad + \left( \frac{128}{3} \textup{Tr}(C^\alpha)L\gamma +16\right) \text{Tr}(C^\alpha) L^2 \gamma^2 + \left(\frac{128}{3}L^2 \gamma^2 + 4\right)\gamma (\tau^2 K^2 + \textup{Tr}(C^{\alpha-2}) \epsilon_\tau^2) \\
& \leq -\frac{1}{4} \int_{k\gamma}^{(k+1)\gamma} \left( \int \left\| C^{\frac{\alpha}{2}}\nabla_{X_t} \log \frac{d\nu_t}{d\mu^{\textbf{y}}}(X_t) \right\|_H^2 d\nu_t (X_t) \right) dt \\
& \quad + \big(4 \sqrt{\text{Tr}(C^\alpha)}+16\big) \text{Tr}(C^\alpha) L^2 \gamma^2 + \left(\frac{1}{\text{Tr}(C^\alpha)}+4\right)\gamma (\tau^2 K^2 + \textup{Tr}(C^{\alpha-2}) \epsilon_\tau^2), 
\end{aligned}
\]
where in the last inequality we invoked $\gamma \leq \frac{1}{\sqrt{128 \text{Tr}(C^\alpha)}L}$. 
Now by averaging over $N>0$ iterations and dropping the negative term, we can derive the result of Theorem \ref{thm:1}:
\[
\begin{aligned}
 & \frac{1}{N\gamma} \int_{0}^{N\gamma} \left( \int \left\| C^{\frac{\alpha}{2}}\nabla_{X_t} \log \frac{d\nu_t}{d\mu^{\textbf{y}}}(X_t) \right\|_H^2 d\nu_t (X_t) \right) dt  \\ & \leq \frac{4 \text{KL}(\nu_0||\mu^{\textbf{y}})}{N\gamma } + \big( \big(16 \sqrt{\text{Tr}(C^\alpha)}+64\big) \text{Tr}(C^\alpha) L^2 \big) \gamma + \underbrace{\left(\frac{4}{\text{Tr}(C^\alpha)}+16\right) K^2}_{A_1} \tau^2 \\ & \quad +\underbrace{\left(\frac{4}{\text{Tr}(C^\alpha)}+16\right) \text{Tr}(C^{\alpha-2})}_{A_2} \epsilon^2_\tau.
\end{aligned}
\]
\end{proof} 


\end{document}